\documentclass[letterpaper,10pt,reqno]{amsart}

\usepackage[utf8]{inputenc} 
\usepackage[T1]{fontenc}    
\usepackage{hyperref}       
\usepackage{url}            
\usepackage{booktabs}       
\usepackage{amsfonts}       
\usepackage{nicefrac}       
\usepackage{microtype}      
\usepackage{xcolor}         
\usepackage{amsmath}
\usepackage{latexsym}
\usepackage{pifont}%
\usepackage{amsbsy}
\usepackage{amssymb}
\def\eqref#1{(\ref{#1})}
\usepackage{graphicx}
\usepackage{subcaption}
\usepackage{enumerate}
\usepackage{enumitem}
\usepackage{mathtools}
\usepackage{color}
\usepackage{float}
\usepackage{makecell}
\usepackage{arydshln}

\usepackage{multirow}
\usepackage[square,sort,compress,comma,numbers]{natbib}

\def\marginpar#1{\ignorespaces}

\newtheorem{theorem}{Theorem}
\newtheorem{lemma}[theorem]{Lemma}

\newtheorem{assump}[theorem]{Assumption}

\numberwithin{equation}{section}
\begin{document}
\title[CDPM]{Contractive Diffusion Probabilistic Models}

\author[Wenpin Tang]{{Wenpin} Tang}
\thanks{Emails: \{wt2319,\,hz2684\}@columbia.edu, Department of Industrial Engineering and Operations Research, Columbia University.}

\author[Hanyang Zhao]{{Hanyang} Zhao}

\date{\today}
\begin{abstract}
Diffusion probabilistic models (DPMs) have emerged as a promising technique in generative modeling. The success of DPMs relies on two ingredients: time reversal of diffusion processes and score matching. 
In view of possibly unguaranteed score matching, we propose a new criterion -- the contraction property of backward sampling in the design of DPMs, leading to a novel class of contractive DPMs (CDPMs). 
Our key insight is that,
the contraction property can provably narrow score matching errors and discretization errors, thus our proposed CDPMs are robust to both sources of error. 
For practical use, we show that CDPM can leverage weights of pretrained DPMs by a simple transformation, and does not need retraining. 
We corroborated our approach by experiments on synthetic 1-dim examples, 
Swiss Roll, MNIST, CIFAR-10 32$\times$32 and AFHQ 64$\times$64 dataset. 
Notably, CDPM steadily improves the performance of baseline score-based diffusion models.
\end{abstract}

\maketitle


\section{Introduction}
\label{sc1}

\quad 
Diffusion probabilistic models (DPMs) have emerged as a 
powerful framework for generative modeling, outperform generative adversarial networks (GANs) on image and audio synthesis \cite{Dh21, Kong21},
and underpins the major accomplishment in text-to-image creators such as 
DALL$\cdot$E 2 \cite{Ramesh22}, Stable Diffusion \cite{Rombach22},
and the text-to-video generator Sora \cite{Sora}.
The concept of DPMs finds its roots in energy-based models \cite{Sohl15},
and is popularized by  \cite{Ho20, Song19, Song20}
in an attempt to produce from noise new samples 
(e.g. images, audio, text)
that resemble the target data, while maintaining diversity.
See \cite{chen2024overview, Tutorial, Yang23} for a review on DPMs.

\quad DPMs relies on forward-backward Markov processes.
{\it The forward process} starts with the target data distribution, 
and runs for some time until the signal is destroyed --
this gives rise to {\em noise} or a non-informative prior. 
{\it The backward process} is then initiated with this prior, 
and reverses the forward process in time to generate samples
whose distribution is close to the target distribution.
In \cite{Austin21, Ho20, Song19}, DPMs are discrete time-indexed Markov chains;
\cite{Chen23, Song21, Song20} model DPMs in continuous time as 
stochastic differential equations (SDEs).
Nevertheless,
there is no conceptual distinction between discrete and continuous DPMs,
as continuous DPMs can be viewed as the continuum limits of discrete DPMs,
and discrete DPMs are time discretizations of the latter.
In this paper, we adopt a {\em continuous-time} score-based perspective,
and algorithms are derived by discretization.

\quad There has been great interest in improving diffusion models further in terms of image/distribution quality, training cost, and generation speed. In this work, we focus on understanding and improving the choice of drift and diffusion coefficients of the governed SDE in the forward process, which will affect the generation quality of the DPMs. Popular examples include
Ornstein-Uhlenbeck (OU) processes \cite{DV21},
variance exploding (VE) SDEs \cite{Song19},
variance preserving (VP) SDEs \cite{Ho20},
and sub-variance preserving (subVP) SDEs \cite{Song20} (see Table \ref{tab:SDEs coefficients} for details). Despite the empirical success of diffusion models driven by these specific SDEs, theoretical understanding of these designs are limited and a fundamental question remains unanswered:
$$\textit{How to design parameters of a DPM for better generation quality?}$$
\quad There has been controversy around the best design of the diffusion model/SDE parameters, e.g. one should choose VP or VE for his own training of diffusion models. According to empirical studies, the design of exploding variance in VE could yield a better result for CIFAR10 and ImageNet synthesis, e.g. in EDM baseline \cite{karras2022elucidating}. However, the reason for this choice is highly empirical and with minimal theoretical justifications. Why VE SDE performs better than VP SDE? The purpose of this paper is to partly answer these questions theoretically and further improve the performance of diffusion models by our theorectical findings. Our proposal is, when facing score matching errors in score-based diffusion models:
$$\textit{The backward process of the DPM shall be \textbf{contractive}.}$$

\quad Here we explain the high-level idea of ``contractive-ness". 
For a diffusion process, ``contractive-ness'' is related to the (long-time) convergence of the process to stationarity. By letting $\mathcal{L}$ be the infinitesimal generator, it is specified by the Lyapunov type condition:
$
\mathcal{L} f \leq-c
$ (or a stronger form $\mathcal{L} f \leq-c f$) outside a compact domain, for some test function $f$ and $c>0$. When adapting to diffusion models, there are two tweaks:
(1) the Lyapunov condition is implicitly satisfied for some diffusion coefficients, and (2) there are score matching errors. 
We further propose a simplified checkable necessary condition on diffusion coefficients to guarantee a contractive backward process, which we elaborate in Section 3.

\quad We call DPMs that yield a contractive backward process as {\em contractive DPMs} (CDPMs).
At a high level, 
the contraction of the backward process will prevent score matching errors,
which may be wild, from expanding over the time.
The contributions of this work are summarized as follows.
\vskip 5 pt
{\bf Methodology}:
We propose a new criterion for designing DPMs.
This naturally leads to a novel class of DPMs,
including contractive OU processes and contractive subVP SDEs.
The idea of requiring the backward process to be contractive
stems from sampling theory of SDEs,
so our methodology is theory-oriented. 
To our best knowledge, 
this is the first paper to integrate contraction into the design of DPMs,
with both theoretical guarantees
and good empirical performance.
\vskip 5 pt
{\bf Theory}:
We prove Wasserstein bounds between contractive DPM samplers
and the target data distribution.
While most previous work (e.g., \cite{Chen23, DV21, LLT22}) 
focused on Kullback–Leibler (KL) or total variation bounds for OU processes,
we consider the Wasserstein metric because  
it has shown to align with human judgment on image similarity \cite{Bor19}, and the standard evaluation metric -- Fr\'echet inception distance (FID) is based on Wasserstein distance.
Early work \cite{DV21, KW22} gave Wasserstein bounds for the existing DPMs (OU processes, VE and VP SDEs)
with exponential dependence on $T$.
This was improved in recent studies \cite{Chen23, LLT23, GNZ23}
under various assumptions of $p_{\scalebox{0.7}{data}}(\cdot)$, 
where the bounds are typically of form:
$$\underbrace{(\small \mbox{noise inaccuracy}) \cdot e^{-T}}_{\tiny \mbox{initialization error}} + \underbrace{(\small \mbox{score mismatch}) \cdot \texttt{Poly}(T)}_{\tiny \mbox{score error}}
+ \underbrace{\texttt{Poly}(\small \mbox{step size}) \cdot \texttt{Poly}(T)}_{\tiny \mbox{discretization error}},$$
with $\texttt{Poly}(\cdot)$ referring to polynomial in the variable. 
Our result gives a Wasserstein bound for CDPMs:
\begin{equation*}
\begin{aligned}
\underbrace{(\small \mbox{noise inaccuracy}) \cdot e^{-T}}_{\tiny \mbox{initialization error}}  + \underbrace{(\small \mbox{score mismatch}) \cdot (1 - e^{-T})}_{\tiny \mbox{score error}}
+ \underbrace{\texttt{Poly}(\small \mbox{step size})}_{\tiny \mbox{discretization error}}.
\end{aligned}
\end{equation*}
Score matching is often trained using blackbox function approximations,
and the errors incurred in this step may be large. 
So CDPMs are designed to be robust to score mismatch
and discretization,
at the cost of possible initialization bias.
\vskip 5 pt
{\bf Experiments}: We apply the proposed CDPMs to both synthetic and real data. In dimension one, we compare contractive OU with OU by adding a fixed noise to the true score function, which yields the same score matching error. Our result shows that contractive OU consistently beats OU, and is robust to different error levels and time discretization. 
We further compare the performance of different models via 
Wasserstein-2 distance of the SWISS Roll dataset and FIDs of MNIST,
which show that CDPMs outperform other SDE models. 
On the task of CIFAR-10 unconditional generation,
we obtain an FID score of 2.47 and an inception score of 10.18 for CDPM, 
which requires no retraining by transforming the pretrained weights of VE-SDE in \cite{Song20}, surpassing all other SDE models.
 
\smallskip
\noindent
{\bf Literature review}.
In the context of generative modeling, 
DPMs were initiated by \cite{Song19} (SMLD) and \cite{Ho20} (DDPM)
using forward-backward Markov chains.
The work \cite{Song20, Song21} unified the previous models 
via a score-based SDE framework,
which also led to deterministic ordinary differential equation (ODE) samplers.
Since then the field has exploded, 
and lots of work has been built upon DPMs and their variants.
Examples include 
DPMs in constrained domains \cite{FK23, LE23, RDD22, DS22},
DPMs on manifolds \cite{Pid22, DeB22, CH23},
DPMs in graphic models \cite{Mei23},
variational DPMs \cite{Kingma21, VK21}
and consistency models \cite{SDC23},
just to name a few.
Early theory \cite{DV21, DeB22, KW22} established the convergence 
of DPMs with exponential dependence on time horizon $T$
and dimension $d$.
Recently, polynomial convergence of various DPMs has been proved
for stochastic samplers \cite{Chen23, Ben23a, LW23, GNZ23, LLT22, LLT23}
and deterministic samplers \cite{Ben23b, Chen23b, Chen23c, LW23},
under suitable conditions on the target data distribution. 

\medskip
\quad The remainder of the paper is organized as follows.
In Section \ref{sc2}, we provide background on score-based DPMs.
Theoretical results for CDPMs are presented in Section \ref{sc3},
and connections to VE are discussed in Section \ref{sc35}.
Experiments are reported in Section \ref{sc4}.
We conclude with Section \ref{sc5}.

\section{Background on Score-based Diffusion Models}
\label{sc2}

\subsection{Forward and Backward SDEs.}
\label{sc21}
In this work, we explain the formulation of diffusion models based on the Stochastic Differential Equations (SDEs) presented by \cite{Song20}. 
Consider a fixed time horizon $T$, and an SDE-governed forward process $\{X_t\}_{t\in [0,T]}$ that describes how data distribution evolves as: 
\begin{equation}
\label{eq:SDE}
\mathrm{d} X_t = b(t, X_t) \mathrm{d}t + \sigma(t) \mathrm{d}B_t, \quad \mbox{with } X_0 \sim p_{\scalebox{0.7}{data}}(\cdot),
\end{equation}
where drift coefficient
$b: \mathbb{R}_+ \times \mathbb{R}^d \to \mathbb{R}^d$ and 
diffusion coefficient $\sigma: \mathbb{R}_+  \to \mathbb{R}_+$
are hyper-parameters as part of the design space of diffusion models.
Theorectically, some conditions on $b(\cdot, \cdot)$ and $\sigma(\cdot)$ are required
so that the SDE \eqref{eq:SDE} is at least well-defined (see \cite[Chapter 5]{KS91}, \cite[Section 3.1]{TZ23} for details).

We further denote $p(t, \cdot)$ be the probability density of $X_t$ (then $p(0,\cdot)=p_{\scalebox{0.7}{data}}(\cdot)$ is the data distribution).
Diffusion Models consider a time reversal $\overline{X}_t: = X_{T-t}$ for $0 \le t \le T$ of the original process $X_t$ for generative sampling. Existing literature on the time reversal of SDEs \cite{Ander82, HP86} have proved that $\overline{X}$ also satisfies an SDE (see proof in
Appendix A.), such that 
\begin{equation}
\label{eqn:SDE_reversal}
\mathrm{d}\overline{X}_t = \bar{b}(t, \overline{X}_t) \mathrm{d}t + \bar{\sigma}(t) \mathrm{d}\overline{B}_t, \quad \mbox{with } \overline{X}_0 \sim p(T, \cdot),
\end{equation}
in which the drift and diffusion coefficient yield a closed-form (we denote $\nabla$ the gradient w.r.p. $x$):
\begin{equation}
\label{eqn:SDE_reversal_coef}
\overline{\sigma}(t) = \sigma(T-t), \quad
\overline{b}(t,x) = -b(T-t, x) + \sigma^2(T-t)  \nabla \log p (T-t,x).
\end{equation}
The term $\nabla \log p(\cdot, \cdot)$ is also known as the Stein's score function. However, scores are unknown, and the procedures or methods to estimate scores are referred as score matching.

\subsection{Score Matching} To estimate the scores, recently score-based diffusion models \cite{Song19, Song20}
proposed to learn $\nabla \log p(t,x)$ by function approximations via deep neural networks. 
More precisely, 
denote by $\{s_\theta(t,x)\}_\theta$ a family of functions on $\mathbb{R}_+ \times \mathbb{R}^d$
parametrized by $\theta$, the scores are approximated by minimizing a weighted $L^2$ loss across time:
\begin{equation}
\label{eq:weighted score matching}
\min_\theta \tilde{\mathcal{J}}_{\text{ESM}}(\theta) = \mathbb{E}_{t \in \mathcal{U}(0, T)}\mathbb{E}_{p(t, \cdot)} \left[\lambda(t)|s_\theta(t,X) - \nabla \log p(t,X)|^2\right].
\end{equation}
where $\mathcal{U}(0, T)$ denotes a uniform distribution on $[0, T]$, and $\lambda: \mathbb{R} \rightarrow \mathbb{R}_{+}$ is a positive weighting function. 
\eqref{eq:weighted score matching} is also known as the {\em explicit score matching} (ESM) objective. For efficiently solving \eqref{eq:weighted score matching} in which scores $\nabla \log p(t,X)$ are still unknown, existing works in denoising diffusion models \cite{Song20,karras2022elucidating} solve an equivalent {\em denoising score matching} (DSM) objective
$$
\tilde{\mathcal{J}}_{\text{DSM}}(\theta)=\mathbb{E}_{t\sim\mathcal{U}(0, T)}\left\{\lambda(t) \mathbb{E}_{X_0\sim p_{data}(\cdot)} \mathbb{E}_{X_t \mid X_0}\left[\left|s_{\theta}(t,X(t))-\nabla_{X_t} \log p(t,X_t \mid X_0)\right|^2\right]\right\}.
$$
DSM objective has the advantage that the conditional distribution $p(t,X_t\mid X_0)$ is commonly tractable, and the scores yield a simple closed form when \eqref{eq:SDE} is a linear SDE (which is a common choice in literature), i.e. $b(t,x)=b(t)\cdot x$ for some function $b(t):\mathbb{R}_{+}\rightarrow \mathbb{R}$. We discuss this equivalence and other score matching methods such as sliced score matching \cite{Song20sl} in Appendix C.

\subsection{Designs of $b$ and $\sigma$.} After learning the scores by parameterization $s_{\theta^*}$, diffusion models can move backward in time from $p(T,\cdot)$ following \eqref{eqn:SDE_reversal}, and sample $\overline{X}_T  \sim p_{\scalebox{0.7}{data}}(\cdot)$ at time $T$. However, here the distribution $p(T, \cdot)$ still depends on the target distribution $p_{\scalebox{0.7}{data}}(\cdot)$ which is not tractable. A non-informative prior is needed to replace $p(T,\cdot)$ for true generative sampling. In addition to the linear SDE design mentioned earlier, two popular choices of diffusion models, VP SDE \cite{Ho20} and VE SDE \cite{Song20}, both ensure the existence of prior. We denote all the priors as $p_{\infty}$ for simplicity following this criterion. We also list the coefficients of popular SDEs (VP, subVP, VE), and the prior distributions in Table \ref{tab:SDEs coefficients} (More details are listed in Appendix B). Especially, VE and VP SDEs are the continuum limits of 
score matching with Langevin dynamics (SMLD) \cite{Song19}
and denoising diffusion probabilistic models (DDPMs) \cite{Ho20} respectively.

\begin{table}[!htbp]
\centering
\resizebox{\textwidth}{!}{
\begin{tabular}{lcccc}
\toprule 
\textbf{SDEs} & \textbf{$b(t)$} & \textbf{$\sigma(t)$}& $p_{\infty}$ & Remarks \\
\midrule
OU  & $ -\theta$ & $\sigma$ & $\mathcal{N}\left(0, I \right)$& $\theta,\sigma>0$\\
VP \cite{Ho20} & $-\frac{1}{2}\beta(t)$ & $\sqrt{\beta(t)}$& $\mathcal{N}(0,I)$ & $\beta(t)= \beta_{\min} + \frac{t}{T}(\beta_{\max} - \beta_{\min})$, with $\beta_{\min} \ll \beta_{\max}$ \\
subVP \cite{Song20} & $-\frac{1}{2}\beta(t)$ & $\sqrt{\beta(t) (1 - e^{-2 \int_0^t \beta(s) ds})}$& $\mathcal{N}(0,I)$ & $\beta(t)= \beta_{\min} + \frac{t}{T}(\beta_{\max} - \beta_{\min})$, with $\beta_{\min} \ll \beta_{\max}$\\
VE \cite{Song20} & $0$ & $\sigma_{\min }\left(\frac{\sigma_{\max }}{\sigma_{\min }}\right)^{\frac{t}{T}} \sqrt{\frac{2}{T} \log \frac{\sigma_{\max }}{\sigma_{\min }}}$ & $\mathcal{N}\left(0, \sigma_{\max }^2 I\right)$ & $\sigma_{\min} \ll  \sigma_{\max}$\\
\bottomrule
\end{tabular}
}
\caption{Design space of $b$ and $\sigma$ for different SDEs.}
\label{tab:SDEs coefficients} 
\end{table}

Finally, score-based diffusion models \cite{Song20} replace $\nabla p(t,x)$ with the matched score $s_{\theta^{*}}(t,x)$ in \eqref{eqn:SDE_reversal}, and also replace the marginal distribution at time $T$ by the non-informative prior $p_{\scalebox{0.7}{noise}}(\cdot)$ to get the backward process for generative sampling:
\begin{equation}
\label{eq:SDErevc}
\mathrm{d}\overline{X}_t = \left(-b(T-t, \overline{X}_t) + \sigma^2(T-t) \, s_{\theta^{*}}(T-t, \overline{X}_t) \right)\mathrm{d}t + \sigma(T-t) \mathrm{d}\overline{B}_t,
\,\, \overline{X}_0 \sim p_{\scalebox{0.7}{noise}}(\cdot).
\end{equation}

\section{Theory for contractive DPMs}
\label{sc3}

\quad In this section,
we formally introduce the idea of CDPMs, and present supportive theoretical results.
For a DPM specified by \eqref{eq:SDE}-\eqref{eq:SDErevc},
we consider the Euler-Maruyama discretization of its backward process $\overline{X}$.
Fix $\delta > 0$ as the step size,
and set $t_k: = k \delta$ for $k = 0, \ldots, N:= \frac{T}{\delta}$.
Let $\widehat{X}_0 = \overline{X}_0$, 
and 
\begin{equation}
\label{eq:EMdisc}
\begin{aligned}
\widehat{X}_k: = \widehat{X}_{k-1} + (-b(T - t_k, \widehat{X}_{k-1}) + 
& \sigma^2(T - t_{k-1}) s_\theta(T - t_{k-1}, \widehat{X}_{k-1})) \delta \\
& + \sigma(T-t_{k-1}) (B_{t_k} - B_{t_{k-1}}), 
\quad \mbox{for } k = 1,\ldots, N.
\end{aligned}
\end{equation}
Our goal is to bound the Wasserstein-2 distance $W_2(p_{\scalebox{0.7}{data}}(\cdot), \widehat{X}_N)$.
Clearly, 
\begin{equation}
\label{eq:triangle}
W_2(p_{\scalebox{0.7}{data}}(\cdot), \widehat{X}_N)
\le W_2(p_{\scalebox{0.7}{data}}(\cdot), \overline{X}_T)
+ \left(\mathbb{E}|\overline{X}_T - \widehat{X}_N|^2\right)^{\frac{1}{2}},
\end{equation}
where the first term on the right side of \eqref{eq:triangle} is
the sampling error at the continuous level,
and the second term is the discretization error.
We will study these two terms in the next two subsections.

\subsection{Sampling error in continuous time}
\label{sc31}

We are concerned with the term $W_2(p_{\scalebox{0.7}{data}}(\cdot), \overline{X}_T)$.
Existing work \cite{Chen23, LLT23, DV21} established $W_2$ bounds 
for OU processes under a bounded support assumption.
Closer to our result (and proof) is the concurrent work \cite{GNZ23}, 
where a $W_2$ bound is derived
for a class of DPMs with $b(t,x)$ separable in $t$ and $x$,
under a strongly log-concavity assumption.

\begin{assump}
\label{assump:sensitivity}
The following conditions hold:
\begin{enumerate}[itemsep = 3 pt]
\item
There exists $r_b: [0,T] \to \mathbb{R}$ such that 
$(x - x') \cdot (b(t,x) - b(t,x')) \ge r_b(t) |x - x'|^2$ for all $t$ and $x,x'$.
\item
There exists $L > 0$ such that $|\nabla \log p(t,x) - \nabla \log p(t,x')| \le L |x-x'|$ for all $t$ and $x,x'$.
\item
There exists $\varepsilon > 0$ such that 
$\mathbb{E}|s_{\theta}(t,\overline{X}_{T-t}) - \nabla \log p(t,\overline{X}_{T-t})|^2 \le \varepsilon^2$ for all $t$.
\end{enumerate}
\end{assump}

\quad The condition (1) assumes the monotonicity of $b(t, \cdot)$
and (2) assumes the Lipschitz property of the score functions. 
In the previous examples, 
$b(t,x)$ is linear in $x$ 
so the density $p(t, \cdot)$ is Gaussian-like,
and its score is almost affine.
Thus, it is reasonable to assume (2).
Conditions (1) and (2) 
are used to quantify how a perturbation of the model parameters in an SDE
affects its distribution. 
The condition (3) specifies how accurate Stein's score is estimated by 
a blackbox estimation. 
There has been work (e.g. \cite{CH23, Koe22, Oko23}) on the efficiency of score approximations.
So it is possible to replace the condition (3) with those score approximation bounds.

\begin{theorem}
\label{thm:sensitivity}
Let Assumption \ref{assump:sensitivity} hold, and any $h > 0$.
Define $\eta:=W_2(p(T, \cdot), p_{\scalebox{0.7}{noise}}(\cdot))$,
and
\begin{equation}
\label{eq:uv}
u(t):=\int_{T-t}^T \left(-2 r_b(s) + (2L + 2h) \sigma^2(s) \right)ds.
\end{equation}
Then we have 
\begin{equation}
\label{eq:sensitivity}
W_2(p_{\scalebox{0.7}{data}}(\cdot), \overline{X}_T) \le \sqrt{\eta^2 e^{u(T)} + \frac{\varepsilon^2}{2h} \int_0^T \sigma^2(t) e^{u(T)-u(T-t)} dt}.
\end{equation}
\end{theorem}

\quad The proof of Theorem \ref{thm:sensitivity} is deferred to Appendix D.
A similar result was given in \cite{KW22} 
under an unconventional assumption that
the score matching functions $s_\theta(t,x)$,
rather than the score functions $\nabla p(t,x)$,
are Lipschitz.
In fact,
the (impractical) assumption that 
the score matching functions are Lipschitz
is not needed at the continuous level,
and can be replaced with the Lipschitz condition on the score functions.
On the other hand,
the Lipschitz property of the score matching functions
are required, for technical purposes,
to bound the discretization error in Section \ref{sc32}.

\quad It is possible to establish sharper bounds under extra (structural) conditions on $(b(\cdot, \cdot), \sigma(\cdot))$, 
and also specify the dependence in dimension $d$ (e.g. \cite{Chen23, GNZ23}).
For instance, if we assume $b(t,x)$ is separable in $t$ and $x$ and linear in $x$,
and $p_{\scalebox{0.7}{data}}(\cdot)$ is strongly log-concave,
then the term $L \sigma^2(s)$ in \eqref{eq:uv} will become $-L' \sigma^2(s)$ for some $L' > 0$. The dependence of $L$ on the $b$ and $\sigma$ can also be readily relaxed given additional assumptions on the data distribution,we leave the full investigation of its theory to the future work.

\quad Now we explain contractive DPMs. 
Revisiting bound \eqref{eq:sensitivity}, 
the sampling error $W_2(p_{\scalebox{0.7}{data}}(\cdot), \overline{X}_T)$
is linear in the score matching error $\varepsilon$
and the initialization error $\eta$,
and these errors may be amplified in time $T$
-- in most aforementioned DPMs, 
$r_b(t) \le 0$ so $u(t)$ is positive and at least linear.
As mentioned earlier,
it is problematic 
if we don't know how good a blackbox score matching $s_\theta(t,x)$ is.
Our idea is simply
 to make $u(t)$ be negative, that is to set $r_b(t) > 0$ sufficiently large, 
in order to prevent the score matching error from propagating in backward sampling.
This yields the class of CDPMs,
which is inherently different from existing DPMs in the sense that
these DPMs often have contractive forward processes,
while our proposal requires contractive backward processes.
Quantitatively, we can set for some $\alpha > 0$,
\begin{equation}
\label{eq:contra1}
\textbf{contractive cond.} \quad \inf_{0 \le t \le T} \left(r_b(t) - (L + h) \sigma^2(t) \right) \ge \alpha.
\end{equation} 
In practice, it suffices to design $(b(\cdot, \cdot), \sigma(\cdot))$ with a positive $r_b(t)$ for ant $t\in [\epsilon,T]$, which we refer as a practical condition (actually a necessary condition) of contractiveness. We formally define {\it CDPMs} as diffusion models governed by SDEs that {\it satisfy this practical version of contractiveness}, i.e. $r_b(t)> 0$. Notice that when $b(t,x)=b(t)\cdot x$, this is equivalent to $b(t)>0$.
We present three examples, contractive OU processes, contractive VP SDEs, and contractive subVP SDEs in Table \ref{tab:Contractive SDEs coefficients} :

\begin{table}[!htbp]
\centering
\resizebox{\textwidth}{!}{
\begin{tabular}{lcccc}
\toprule 
\textbf{SDEs} & \textbf{$b(t)$} & \textbf{$\sigma(t)$}& $p_{\infty}$ & Remarks \\
\midrule 
COU  & $ \theta$ & $\sigma$ & $\mathcal{N}\left(0,  
\frac{\sigma^2}{2\theta} ( e^{2 \theta T} -1) I \right)$& $\theta>0$\\
CVP  & $\frac{1}{2}\beta(t)$ & $\sqrt{\beta(t)}$& $\mathcal{N}\left( 0, 
(e^{\frac{T}{2} (\beta_{\max} + \beta_{\min})}-1 ) I \right)$ & $\beta(t)= \beta_{\min} + \frac{t}{T}(\beta_{\max} - \beta_{\min})$, with $\beta_{\min} \ll \beta_{\max}$ \\
CsubVP  & $\frac{1}{2}\beta(t)$ & $\sqrt{\beta(t) (e^{2 \int_0^t \beta(s) ds}-1)}$& $\mathcal{N}\left( 0, 
(e^{\frac{T}{2} (\beta_{\max} + \beta_{\min})}-1 )^2 I \right)$ & $\beta(t)= \beta_{\min} + \frac{t}{T}(\beta_{\max} - \beta_{\min})$, with $\beta_{\min} \ll \beta_{\max}$\\

\bottomrule
\end{tabular}
}
\caption{Design space of $b$ and $\sigma$ for CDPMs.}
\label{tab:Contractive SDEs coefficients} 
\end{table}

\quad To illustrate the benefits of contractive-ness, we further give a bound for CVP. As in Table \ref{tab:Contractive SDEs coefficients}, the backward process of CVP SDE is:
\begin{equation}
\label{eq:CVP}
\begin{aligned}
d \overline{X}_t  = & \left(-\frac{1}{2} \beta(T-t) \overline{X}_t  +\beta(T-t) \nabla \log p(T-t, \overline{X}_t)\right)  dt \\
& \qquad \qquad \qquad+ \sqrt{\beta(T-t)} dB_t, \quad 
\overline{X}_0 \sim \mathcal{N}\left( 0, 
(e^{\frac{T}{2} (\beta_{\max} + \beta_{\min})}-1 ) I \right).
\end{aligned}
\end{equation}
Recall that a function $\ell: \mathbb{R}^d \to \mathbb{R}$ is $\kappa$-strongly concave if $(\nabla \ell(x) - \nabla \ell(y)) \cdot (x - y) \le - \kappa|x-y|^2$.
\begin{theorem}
\label{thm:CVPW2}
Let $(\overline{X}_t, \, 0 \le t \le T)$ be specified by \eqref{eq:CVP} (the backward process of CVP).
Assume that 
$\log p_{\scalebox{0.7}{data}}(\cdot)$ is $\kappa$-strongly log-concave,
and $\mathbb{E}_{p_{\scalebox{0.7}{data}}(\cdot)}|x|^2 < \infty$.
For $h<\min(\frac{1}{2},\frac{1}{\beta_{\text{max}}T}\frac{\kappa}{1+\kappa})$, we have:
\begin{equation}
\label{eq:sensitivity3}
W^2_2(p_{\scalebox{0.7}{data}}(\cdot), \overline{X}_T) \le 
e^{-2\left(\frac{\kappa}{1+\kappa} - \beta_{\max} hT + \mathcal{O}(e^{-\beta_{\min} T})  \right)}\mathbb{E}_{p_{\tiny \mbox{data}}(\cdot)}|x|^2 + \frac{\varepsilon^2}{2h (1-2h)}.
\end{equation}
\end{theorem}

\quad The proof of Theorem \ref{thm:CVPW2} is given in Appendix E.
It is easy to see from the theorem that CVP (which follows similarly for other CDPMs) 
allow to control the score matching error $\varepsilon$,
at the possible cost of initialization error coming from $\eta$.
Note that if $\beta_{\max} T$ is asymptotically small, 
this error is bounded.
This requires scaling the hyperparameters with respect to $T$ in the model.
We observe in the experiment that tuning a moderate level $\beta$ is important to let CDPM benefit from contraction while not suffer from the initialization error.
Also note that it is not necessary to send $T \to \infty$, and 
$T = 1$ is taken in \cite{Song20}.

\subsection{Discretization error}
\label{sc32}

We study the discretization error $\left(\mathbb{E}|\overline{X}_T - \widehat{X}_N|^2\right)^{\frac{1}{2}}$
for CDPMs.
Classical SDE theory \cite{KP92} indicates that $\left(\mathbb{E}|\overline{X}_T - \widehat{X}_N|^2\right)^{\frac{1}{2}} \le C(T) \delta$,
with the constant $C(T)$ exponential in $T$.
Here we show that by a proper choice of the pair $(b(\cdot, \cdot), \sigma(\cdot))$ leading to CDPMs, 
the constant $C(T)$ can be made independent of $T$.
In other words, the discretization error will not expand over the time.
We need some technical assumptions. 
\begin{assump}
\label{assump:2dis}
The following conditions hold:
\begin{enumerate}[itemsep = 3 pt]
\item 
There exists $L_\sigma > 0$ such that $|\sigma(t) - \sigma(t')| \le L_\sigma |t-t'|$ for all $t, t'$.
\item 
There exists $R_\sigma > 0$ such that $\sigma(t) \le R_\sigma$ for all $t$. 
\item 
There exists $L_b > 0$ such that $|b(t,x) - b(t',x')| \le L_b  (|t-t'| + |x - x'|)$ for all $t,t'$ and $x, x'$.
\item 
There exists $L_s > 0$ such that $|s_\theta(t,x) - s_\theta(t',x')| \le L_s (|t -t'| + |x- x'|)$ for all $t,t'$ and $x,x'$.
\item 
There exists $R_s > 0$ such that $|s_\theta(T,x)| \le R_s (1 + |x|)$ for all $x$.
\end{enumerate}
\end{assump}

\quad Next we introduce a contractive assumption that is consistent with \eqref{eq:contra1}.

\begin{assump}
\label{assump:contraction}
There exists $\beta > 0$ such that
\begin{equation}
\label{eq:condcontraction}
\int_{T-t}^T (r_b(s) - L_s \sigma^2(s))\, ds \ge \beta t, \quad \mbox{for all }t,
\end{equation}
or simply 
\begin{equation}
\label{eq:contral2}
\beta: = \inf_{0 \le t \le T} \left(r_b(t) - L_s \sigma^2(t)\right) > 0.
\end{equation}
\end{assump}

\begin{theorem}
\label{thm:globalerr}
Let Assumptions \ref{assump:sensitivity}, \ref{assump:2dis} and \ref{assump:contraction} hold.
Then there exists $C > 0$ (independent of $\delta, T$) such that
for $\delta >0$ sufficiently small,
\begin{equation}
\label{eq:globalerr}
\left(\mathbb{E}|\overline{X}_T - \widehat{X}_N|^2 \right)^{\frac{1}{2}} \le C \sqrt{\delta}.
\end{equation}
\end{theorem}

\quad The proof of Theorem \ref{thm:globalerr} is given in Appendix F. 


\section{Connections between CDPM and VE}
\label{sc35}
\quad In this section, we draw connections between CDPM and VE.
We first show that VE exhibits some hidden contractive property. 
Then we show how to exploit the pretrained models such as VE 
to achieve CDPM, which does not require retraining. 

\subsection{VE is contractive at earlier denoising steps}
We illustrate with an example that the backward process of VE
also yields the contractive property at earlier stages of the denoising process. 
In Figure \ref{fig:Contractive Illustration}, the angles correspond to the scores of a normal distribution with mean 0 (the case of the VE prior).
We see that the two points becomes closer after a denosing step ,
which provides an explanation of the hidden contraction. 
However, VE may lose this contractive property 
when the distribution is close to the target data distribution. The score matching error and discretization error near $t\approx T-\epsilon$ in the backward process indeed plays a large impact as in the empirical studies \cite{karras2022elucidating}, which motivates the design of CDPM.

\begin{minipage}{\textwidth}
\vspace{3 pt}
\begin{minipage}[b]{0.45\textwidth}
    \centering
    \includegraphics[width=0.95\linewidth]{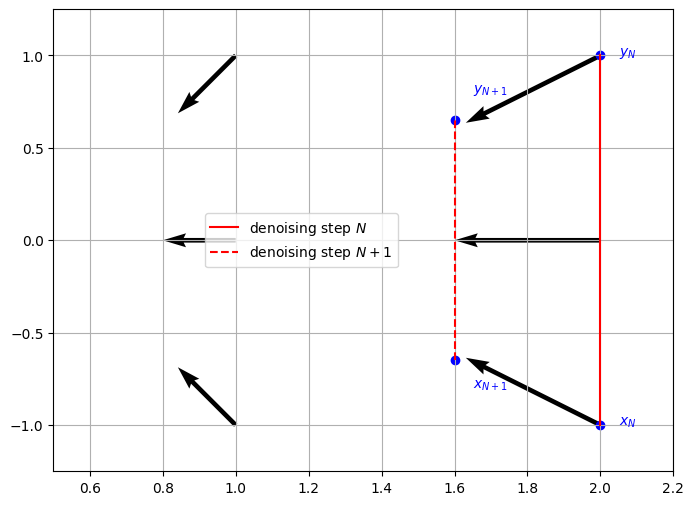}
    \captionof{figure}{Contraction of VE}
    \label{fig:Contractive Illustration}
  \end{minipage}
\hfill
\begin{minipage}[b]{0.54\textwidth}
    \small
    \centering
\includegraphics[width=0.95\linewidth]{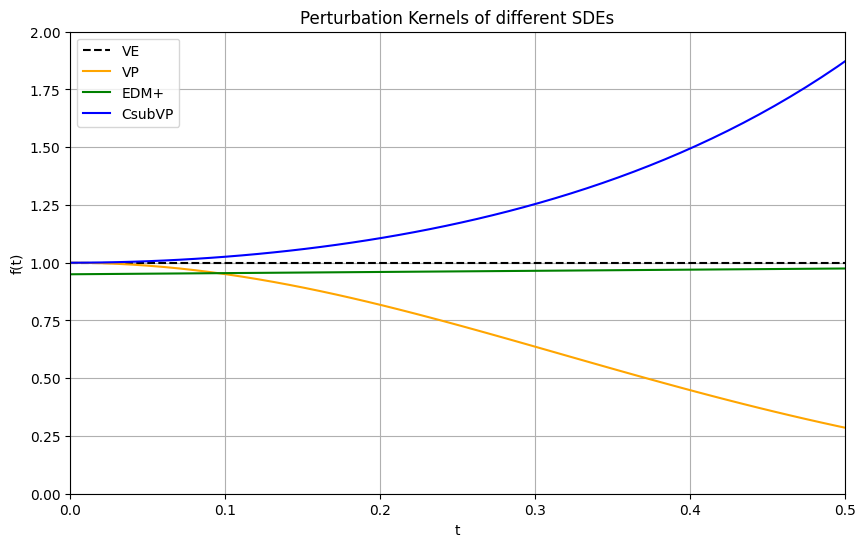}
\captionof{figure}{Perturbation kernels}
\label{fig:perturbation kernels of different SDEs}
\end{minipage}
\vspace{2 pt}
\end{minipage}

\subsection{CDPM can be transformed from VE}
We show that CDPM, though different from existing DPMs,
can be derived from VE via a time/space change of variable, 
so does not require pretraining the score matching objectives again. 
Note that for COU, CVP and CsubVP, 
the parameter $b(t,x)$ is separate in $t$ and $x$.
We follow \cite{karras2022elucidating},
and define the perturbation kernel of CDPM as:
\begin{equation*}
X_t \mid X_0\sim \mathcal{N}\left(f(t) X_0,\, f(t)^2 g(t)^2 I\right),
\end{equation*}
where 
\begin{equation}
f(t)=e^{\frac{t^2}{4T}(\beta_{\max} - \beta_{\min}) + \frac{t}{2} \beta_{\min}} \quad \text{and} \quad g(t)=f(t)-f^{-1}(t).
\end{equation}
Figure \ref{fig:perturbation kernels of different SDEs} plots
different perturbation kernels of SDE models:
existing models lead to either a constant or a decreasing kernel,
while we propose the kernel be increasing.
This yields CsubVP and EDM+, 
which we will show in Section \ref{sc43}.

\quad Denote by $p_{\scalebox{0.7}{VE}}(t,\cdot)$ and $p_{\scalebox{0.7}{CsubVP}}(t,\cdot)$ 
the probability distribution of $X_t$ following VE and CsubVP respectively.
Assume that we have access to a pretrained VE score matching
$s_{pre}(t,x)\approx \nabla \log p_{VE}(t,x)$.
We can then compute the CDPM score by the following transformation,
which is read from \cite[Equation (12) and (19)]{karras2022elucidating}.

\begin{theorem}
Assume that $\sigma^2_{\max}-\sigma^2_{\min}>g^2(T)$.
We have for $t\in [0,T]$, 
\begin{equation}
p_{\scalebox{0.7}{CsubVP}}(t,x) = f(t)^{-d} p_{\scalebox{0.7}{VE}}(\tau(t),x/f(t)),
\end{equation}
where
\begin{equation}
\tau(t)=\frac{T}{2}\frac{\log(1+\frac{g^2(t)}{\sigma^2_{\min}})}{\log(\sigma_{\max}/\sigma_{\min})}.
\end{equation}
\end{theorem}
\quad So it suffices to take $\nabla \log p_{\scalebox{0.7}{CsubVP}}(t,x)\approx s_{pre}(\tau(t),x/f(t))$,
meaning that we can exploit existing score matching neural nets, 
or pretrained weights for CDPM sampling.

\section{Experiments}
\label{sc4}
\quad In this section, we report empirical results on the proposed contractive approach
and CDPMs.
We conduct experiments on a 1-dimensional synthetic example, Swiss Roll, MNIST, CIFAR10 32$\times$32 and AFHQv2 64$\times$64 datasets.

\subsection{CDPM shows better performance with the same scoring matching error} 
\label{sc41}
The goal is to learn/generate a single point mass at $x_0=-1$ in dimension one.
Since we can compute the score explicitly, 
we test the performance of different SDE models by adding the SAME noise level of error at each time/point.
Implementation details are given in Appendix G.1.
Table \ref{tab:global} compares the $W_2$ errors for OU and COU, 
with different noise levels and time discretization.
As is expected from the theory, COU is more robust to score matching error and time discretization.
\begin{table}[ht]
    \vspace{-10 pt}
    \centering
    \begin{minipage}{0.45\textwidth}
        \centering
        \begin{tabular}{l c c}
            \toprule
            Noise Level $\backslash$ $W_2\downarrow$ & $OU$ & $COU$ \\
            \midrule
            $\epsilon = 0.02$ ~& $0.245$ & $0.22$\\
            $\epsilon = 0.05$ & $0.265$ & $0.227$\\
            $\epsilon = 0.1$ & $0.30$ & $0.23$\\
            $\epsilon = 0.2$ & $0.39$ & $0.25$\\
            $\epsilon = 0.5$ & $0.7$ & $0.42$\\
            $\epsilon = 1$ & $1.3$ & $0.8$\\
            \bottomrule
        \end{tabular}
        \subcaption{time discretization $\Delta t= 0.02$}
        \label{tab:dt02}
    \end{minipage}\hfill
    \begin{minipage}{0.45\textwidth}
        \centering
        
        \begin{tabular}{l c c}
            \toprule
            Noise Level $\backslash W_2\downarrow$ & $OU$ & $COU$ \\
            \midrule
            $\epsilon = 0.02$ ~& $0.41$ & $0.35$\\
            $\epsilon = 0.05$ & $0.44$ & $0.36$\\
            $\epsilon = 0.1$ & $0.48$ & $0.36$\\
            $\epsilon = 0.2$ & $0.58$ & $0.36$\\
            $\epsilon = 0.5$ & $0.92$ & $0.43$\\
            $\epsilon = 1$ & $1.5$ & $0.7$\\
            \bottomrule
        \end{tabular}
        \subcaption{time discretization $\Delta t= 0.05$}
        \label{tab:dt05}
    \end{minipage}
    \caption{$W_2$ distance under the same score matching error.}
    \vspace{3 pt}
    \label{tab:global}
\end{table}

\subsection{Swiss Roll and MNIST datasets}
\label{sc42}
We apply CsubVP to Swiss Roll and MNIST datasets. 
Implementation details are reported in Appendix G.2.
Figure \ref{fig:Swiss Roll} shows the evolution process of CsubVP on the Swiss Roll data.
Figure \ref{fig:MNIST} provides image synthesis by CsubVP on MNIST. 
Table \ref{tab:w2_metric} shows a clear advantage of CDPMs 
over other SDE models in terms of $W_2$ error and FID score.


\begin{table}[!htbp]
    \begin{minipage}{0.4\textwidth}
    \centering
        \begin{tabular}{l c c}
            \toprule
            Model (SDE) & $W_2$ $\downarrow$ & FIDs $\downarrow$ \\
            \midrule
            OU ~& $0.29$ & -\\
            VP~\cite{Song20}& $0.33$ & $0.79$ \\
            subVP~\cite{Song20}& $0.34$ & $0.52$ \\
            VE~\cite{Song20}& $0.18$ &  $0.20$ \\
            \midrule
            \textbf{CDPMs} \\
            \midrule
            COU & $\textbf{0.10}$ & -\\
            CsubVP & $\textbf{0.14}$ & $\textbf{0.03}$\\
            \bottomrule
        \end{tabular}
        \vspace{5 pt}
        \caption{$W_2$ metric on Swiss Roll and FIDs on MNIST synthesis.}
        \vspace{3 pt}
        \label{tab:w2_metric}
    \end{minipage}
    \begin{minipage}{0.5\textwidth}
        \centering
        \begin{subfigure}{1\textwidth}
        \includegraphics[width=1\linewidth]{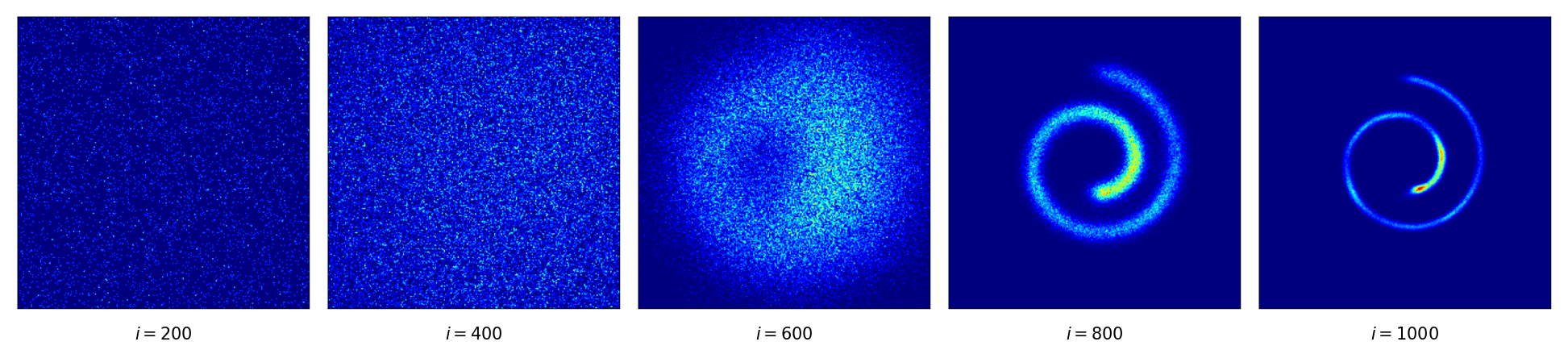}
        \subcaption{Swiss Roll generation with 200, 400, 600, 800, 10000 iterations.}
        \label{fig:Swiss Roll}
        \end{subfigure}
        \begin{subfigure}{0.8\textwidth}
        \includegraphics[width=1\linewidth]{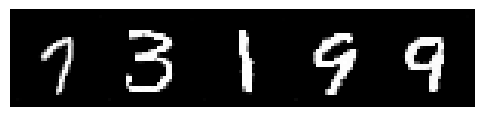}
        \subcaption{MNIST synthesis by CsubVP.}
        \label{fig:MNIST}
        \end{subfigure}
     \end{minipage}
\end{table}

\subsection{CIFAR-10 dataset}
\label{sc43}
We first test the performance of our proposed CsubVP on the task of unconditional synthesis of the CIFAR-10 dataset. 
We compute and compare FID and inception scores of CsubVP and other SDE models. 
Implementation details are given in Appendix G.3.

\quad Figure \ref{fig:CIFAR10syn1} provides image synthesis 
on CIFAR-10.
From Table \ref{tab:cifar},
CsubVP shows the best performance
among all known classes of SDE-based diffusion models.
In particular, it outperforms VE SDEs (non-deep version in \cite{Song20})
by achieving both smaller FIDs and higher Inception Scores.
($*$ the model evaluation is conducted on our own machine 
(4 4090RTX GPUs)
given the checkpoints provided by \cite{Song20}).

\begin{minipage}{\textwidth}
\vspace{5 pt}
\begin{minipage}[b]{0.49\textwidth}
    \centering
    \includegraphics[width=0.95\linewidth]{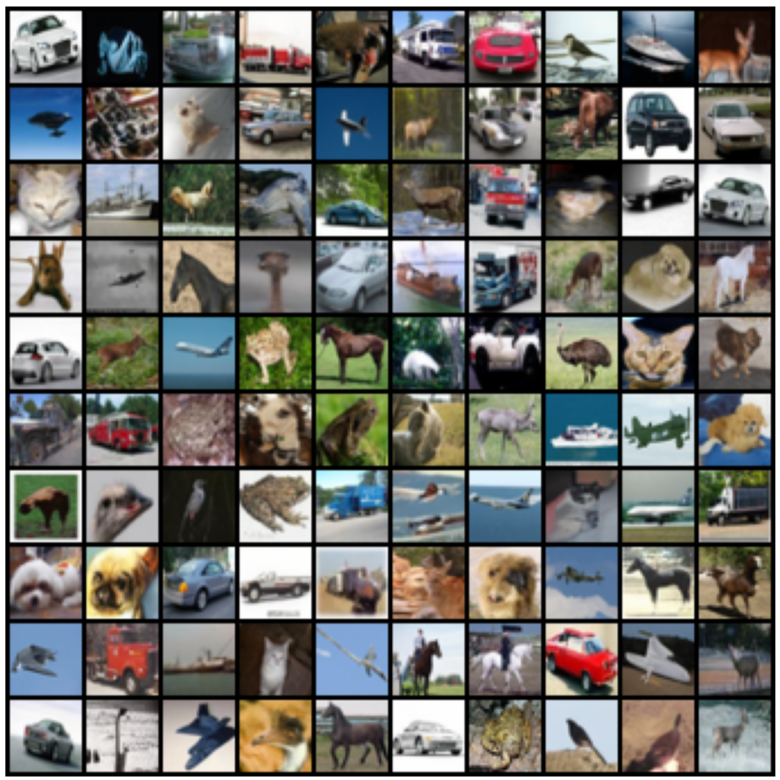}
    \captionof{figure}{CIFAR-10 Synthesis (CsubVP).}
    \label{fig:CIFAR10syn1}
  \end{minipage}
\hfill
\begin{minipage}[b]{0.49\textwidth}
    \small
    \centering
    \begin{tabular}{l c c c c}
    \toprule
        Model & Inception $\uparrow$  &FID $\downarrow$\\
         \midrule
        PixelCNN~\cite{van2016conditional} & 4.60 &65.9\\
        IGEBM~\citep{du2020improved}& 6.02 &40.6\\
        ViTGAN~\cite{lee2021vitgan} & $9.30$ &$6.66$\\
        StyleGAN2-ADA~\citep{karras2020training}& $9.83$ & $2.92$\\
        NCSN~\cite{Song19}&${8.87 }$ & $25.32$\\
        NCSNv2~\cite{song2020improved}&${8.40 }$ & $10.87$\\
        DDPM~\citep{Ho20}&$9.46$ & $3.17$\\
        DDIM, $T=50$~\citep{DDIM} &-&$4.67$\\
        DDIM, $T=100$~\citep{DDIM} &-&$4.16$\\
        \midrule
        \textbf{NCSN$++$}\\
        \midrule
        VP SDE~\cite{Song20}& $9.58$ &  $2.55$\\
        subVP SDE~\cite{Song20}& $9.56$ & $2.61$\\
         VE SDE~\cite{Song20}& $9.68^*$ &  $2.50^*$\\
         CsubVP ~ & $\textbf{10.18}$ & $\textbf{2.47}$ &\\
         \bottomrule
\end{tabular}
\vspace{2.5pt}
\captionof{table}{Inception \& FID.}
\label{tab:cifar}
\end{minipage}
\end{minipage}

\smallskip
\quad We also show how to improve the baseline pretrained models using our idea of contraction. 
Given the checkpoints of EDM \cite{karras2022elucidating} on CIFAR10 and AFHQv2 datasets, 
we modify the perturbation kernel to let $s(0)=1-\epsilon<1$,
leading to an increasing function as in Figure \ref{fig:perturbation kernels of different SDEs}.
This simple technique, motivated by our contraction approach,
yields improvement to the EDM baselines as in Table \ref{tab:cifar2}.
Moreover, we observe the improvement of the sample quality
by comparing the images generated by EDM and EDM with contraction,
see Figure \ref{fig:CIFAR10syn}.
\vspace{10 pt}

\begin{minipage}{\textwidth}
\begin{minipage}[b]{0.4\textwidth}
\vspace{2pt}
    \centering
    \includegraphics[width=1\linewidth]{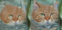}
    \captionof{figure}{(AFHQv2 sample) \textbf{LEFT}: EDM,
    \textbf{RIGHT}: EDM with contraction.}
    \label{fig:CIFAR10syn}
\end{minipage}
\hfill
\begin{minipage}[b]{0.58\textwidth}
    \small
    \centering
    \begin{tabular}{l c c c}
    \toprule
        Model/FID $\downarrow$ & EDM*\cite{karras2022elucidating}  & `+' Contraction & NFE\\
         \midrule
        \textbf{CIFAR10 32$\times$32}\\
        \midrule
        VP SDE (cond)& 1.85 & 1.83 & 35 \\
        VE SDE (cond)& 1.83 & 1.82 & 35 \\
        VP SDE (uncond)& 1.96 & 1.94 & 35 \\
        VE SDE (uncond) & 1.97 & 1.97& 35 \\
        \midrule
        \textbf{AFHQv2 64$\times$64}\\
        \midrule
        VP SDE (uncond)& 2.10 & 2.08 & 79 \\
        VE SDE (uncond) & 2.24 & 2.20 & 79 \\
        \bottomrule
\end{tabular}
\vspace{-2.5pt}
\captionof{table}{FID scores (*our reruns).}
\label{tab:cifar2}
\end{minipage}
\end{minipage}

\section{Conclusion}
\label{sc5}

\quad In this paper,
we propose a new criterion -- 
the contraction of backward sampling
in the design of SDE-based DPMs.
This naturally leads to a novel class of contractive DPMs.
The main takeaway is that the contraction of the backward process
limits score matching errors from propagating,
and controls discretization error as well. 
Our proposal is supported by theoretical considerations,
and is corroborated by experiments. 
Notably, our proposed contractive subVP outperforms
other SDE-based DPMs on CIFAR 10 dataset.
Though the intention of this paper is not to beat SOTA diffusion models,
CDPMs show promising results
that we hope to trigger further research.

\quad There are a few directions to extend this work.
First, we assume that score matching errors
are bounded in $L^2$
as in \cite{Chen23, GNZ23, LLT23}.
It is interesting to see whether this assumption 
can be relaxed to more realistic conditions 
given the application domain. 
Second, it is desirable to establish sharp theory for CDPMs, with dimension dependence.
Finally, our formulation is based on SDEs,
and hence stochastic samplers.
We don't look into ODE samplers as in 
\cite{Song20,xu2022poisson,LW23, Chen23b}.
This leaves open the problems such as
whether the proposed CDPMs perform
well by ODE samplers,
and why the ODE samplers derived from SDEs
outperform those directly learnt,
as observed in previous studies.

\medskip
{\bf Acknowledgement}: 
We thank Jason Altschuler, Yuxin Chen, Sinho Chewi, Xuefeng Gao, Dan Lacker, Yuting Wei and David Yao for helpful discussions.
We gratefully acknowledges financial support through NSF grants DMS-2113779 and
DMS-2206038,
and by a start-up grant at Columbia University.

\bibliographystyle{abbrv}
\bibliography{diffusion}

\clearpage
\appendix
\section*{Appendix}

\subsection*{A. Proof of SDE time reversal}

The time reversal of SDEs \cite{Ander82, HP86}.
\begin{theorem}
\label{thm:SDErev}
Under suitable conditions on $b(\cdot, \cdot)$, $\sigma(\cdot)$ and $\{p(t, \cdot)\}_{0 \le t \le T}$, 
we have
\begin{equation}
\label{eq:coefrev}
\overline{\sigma}(t) = \sigma(T-t), \quad
\overline{b}(t,x) = -b(T-t, x) + \sigma^2(T-t)  \nabla \log p (T-t,x),
\end{equation}
where the term $\nabla \log p(\cdot, \cdot)$ is called Stein's score function.
\end{theorem}

\begin{proof}
Here we give a heuristic derivation of the time reversal formula \eqref{eq:coefrev}.
First, the infinitesimal generator of $X$ is 
$\mathcal{L}:= \frac{1}{2} \sigma^2(t) \Delta  +b \cdot \nabla$.
It is known that the density $p(t,x)$ satisfies the the Fokker–Planck equation:
\begin{equation}
\label{eq:FPfd}
\frac{\partial}{\partial t} p(t,x) = \mathcal{L}^*p(t,x),
\end{equation}
where $\mathcal{L}^*:= \frac{1}{2} \sigma^2(t) \Delta - \nabla \cdot b$ is the adjoint of $\mathcal{L}$.
Let $\overline{p}(t,x) := p(T-t, x)$ be the probability density of the time reversal $\overline{X}$. 
By \eqref{eq:FPfd}, we get
\begin{equation}
\label{eq:FPfdbar}
\frac{\partial}{\partial t} \overline{p}(t,x) = -\frac{1}{2} \sigma^2 (t) \Delta \overline{p}(t,x)  + \nabla \cdot \left(b(T-t, x) \, \overline{p}(t,x) \right).
\end{equation}
On the other hand, we expect the generator of $\overline{X}$ to be 
$\overline{\mathcal{L}}:=  \frac{1}{2} \overline{\sigma}^2(t)  \Delta + \overline{b} \cdot \nabla$.
The Fokker-Planck equation for $\overline{p}(t,x)$ is
\begin{equation}
\label{eq:FPbdbar}
\frac{\partial}{\partial t} \overline{p}(t,x) = \frac{1}{2}  \overline{\sigma}^2(t) \Delta \overline{p}(t,x)- \nabla \cdot \left(\overline{b}(t, x) \, \overline{p}(t,x) \right).
\end{equation}
Comparing \eqref{eq:FPfdbar} and \eqref{eq:FPbdbar},
we set $\overline{\sigma}(t) = \sigma(T-t)$ and then get
\begin{equation*}
\left(b(T-t,x) + \overline{b}(t,x)\right) \overline{p}(t,x) = \sigma^2(T-t) \, \nabla \overline{p}(t,x).
\end{equation*}
This yields the desired result. 
\end{proof}

\quad Let's comment on Theorem \ref{thm:SDErev}.
\cite{HP86} proved the result by assuming that $b(\cdot, \cdot)$ and $\sigma(\cdot, \cdot)$ are globally Lipschitz, and 
the density $p(t,x)$ satisfies an a priori $H^1$ bound. 
The implicit condition on $p(t,x)$ is guaranteed if $\partial_t + \mathcal{L}$ is hypoelliptic.
These conditions were relaxed in \cite{Quastel02}. 
In another direction, \cite{Fo85, Fo86} used an entropy argument to prove the time reversal formula in the case $\sigma(t) = \sigma$. 
This approach was further developed in \cite{CCGL21} which made connections to optimal transport. 

\subsection*{B. Examples of DPMs}

As in Table \ref{tab:SDEs coefficients}, we present a few examples of DPMs, here are their concrete details:

\begin{enumerate}[itemsep = 3 pt]
\item[(a)]
OU processes: $b(t,x) = \theta (\mu - x)$ with $\theta > 0$, $\mu \in \mathbb{R}^d$; 
$\sigma(t) = \sigma > 0$.
The distribution of $(X_t \,|\, X_0 = x)$ is 
$\mathcal{N}(\mu + (x - \mu) e^{-\theta t}, \frac{\sigma^2}{2\theta} (1 - e^{-2 \theta t}) I)$,
which is approximately 
$\mathcal{N}(\mu, \frac{\sigma^2}{2 \theta} I)$ as $t$ is large.
The backward process specializes to
\begin{equation}
\label{eq:OUrev2}
d\overline{X}_t = (\theta(\overline{X}_t - \mu) + \sigma^2 \nabla \log p(T-t, \overline{X}_t)) dt  + \sigma dB_t, 
\,\, \overline{X}_0 \sim \mathcal{N}\left(\mu, \frac{\sigma^2}{2 \theta} I\right).
\end{equation}
\item[(b)]
VE-SDE:
$b(t,x) = 0$ and 
$\sigma(t) = \sigma_{\min} \left(\frac{\sigma_{\max}}{\sigma_{\min}} \right)^{\frac{t}{T}} \sqrt{\frac{2}{T} \log \frac{\sigma_{\max}}{\sigma_{\min}}}$
with 
$\sigma_{\min} \ll  \sigma_{\max}$.
The distribution of $(X_t \,|\, X_0 = x)$ is $\mathcal{N}\left(x, \sigma^2_{\min}\left( \left(\frac{\sigma_{\max}}{\sigma_{\min}} \right)^{\frac{2t}{T}}- 1 \right) I \right)$, 
which can be approximated by $\mathcal{N}(0, \sigma^2_{\max} I)$ at $t = T$.
The backward process is:
\begin{equation}
\label{eq:VErev2}
d\overline{X}_t  = \sigma^2(T-t) ) \nabla \log p(T-t, \overline{X}_t)+  \sigma(T-t) dB_t, 
\,\, \overline{X}_0 \sim \mathcal{N}(0, \sigma^2_{\max} I).
\end{equation}
\item[(c)]
VP-SDE: 
$b(t,x) = -\frac{1}{2}\beta(t) x$ and $\sigma(t) = \sqrt{\beta(t)}$,
where 
$\beta(t):= \beta_{\min} + \frac{t}{T}(\beta_{\max} - \beta_{\min})$ with $\beta_{\min} \ll \beta_{\max}$.
The distribution of $(X_t \,|\, X_0 = x)$ is
$$\mathcal{N}( e^{-\frac{t^2}{4T}(\beta_{\max} - \beta_{\min}) - \frac{t}{2} \beta_{\min}}  x,  
(1 -e^{-\frac{t^2}{2T}(\beta_{\max} - \beta_{\min}) - t \beta_{\min}} ) I),$$
which can be approximated by $\mathcal{N}(0,I)$ at $t = T$.
The backward process is:
\begin{equation}
\label{eq:VPrev2}
\begin{aligned}
d\overline{X}_t & = \bigg(\frac{1}{2} \beta(T-t) \overline{X}_t + \beta(T-t) \nabla \log p(T-t, \overline{X}_t)\bigg)  dt \\
& \qquad \qquad \qquad \qquad \qquad \qquad  \quad + \sqrt{\beta(T-t))} dB_t, \,\, \overline{X}_0 \sim \mathcal{N}(0, I).
\end{aligned}
\end{equation}
\item[(d)]
subVP-SDE: 
$b(t,x) = -\frac{1}{2} \beta(t)x$ and $\sigma(t) = \sqrt{\beta(t) (1 - e^{-2 \int_0^t \beta(s) ds})}$.
The distribution of $(X_t \,|\, X_0 = x)$ is
$\mathcal{N}( e^{-\frac{t^2}{4T}(\beta_{\max} - \beta_{\min}) - \frac{t}{2} \beta_{\min}}  x,  
(1 -e^{-\frac{t^2}{2T}(\beta_{\max} - \beta_{\min}) - t \beta_{\min}} )^2 I )$,
which can be approximated by $\mathcal{N}(0,I)$ at $t = T$.
The backward process is:
\begin{equation}
\label{eq:SVPrev2}
\begin{aligned}
d\overline{X}_t & = \bigg( \frac{1}{2} \beta(T-t) \overline{X}_t  + \beta(T-t)(1 - \gamma(T-t)) \nabla \log p(T-t, \overline{X}_t)\bigg)  dt \\
& \qquad \qquad \qquad \qquad + \sqrt{\beta(T-t) (1 - \gamma(T-t))} dB_t,  \,\, \overline{X}_0 \sim \mathcal{N}(0, I),
\end{aligned}
\end{equation}
where $\gamma(t):= e^{-2 \int_0^t \beta(s) ds} = e^{-\frac{t^2}{T}(\beta_{\max} - \beta_{\min}) - 2t \beta_{\min}}$.
\end{enumerate}
\subsection*{C. Score matching}



\smallskip
(a) Sliced score matching. 
One way is that we further address the term $\nabla_x\cdot s_{\theta}(t,x)$ by random projections. The method proposed in \cite{Song20sl} is called sliced score matching. 
Considering the Jacobian matrix $\nabla s_{\theta}(t,x)\in\mathbb{R}^{d \times d}$, 
we have
$$
\nabla \cdot s_{\theta}(t,x)=\operatorname{Tr}(\nabla s_{\theta}(t,x))=\mathbb{E}_{v\sim \mathcal{N}(0,I)}\left[v^{\top}\nabla s_{\theta}(t,x)v\right].
$$
We can then rewrite the training objective as:
\begin{equation}
\label{eq:equivscorematb_slicedSM}
\min_\theta \widetilde{\mathcal{J}}_{\text{SSM}}(\theta) = \mathbb{E}_{t \in \mathcal{U}(0, T)}\mathbb{E}_{v_t\sim \mathcal{N}(0,I)}\mathbb{E}_{p(t, \cdot)} \left[\lambda(t)\left(\|s_\theta(t,X)\|^2  + 2 \, v^{\top}\nabla (v^{\top}s_{\theta}(t,x))\right)\right].
\end{equation}
which can be computed easily. 
It requires only two times of back propagation, as $v^{\top}s_{\theta}(t,x)$ can be seen as adding a layer of the inner product between $v$ and $s_{\theta}$.

\medskip
(b) Denoising score matching. 
The second way is that we go back to the objective \eqref{eq:weighted score matching},
and use a nonparametric estimation.
The idea stems from \cite{Hyv05,Vi11},
in which it was shown that 
$\mathcal{J}_{ESM}$ is equivalent to the following denoising score matching (DSM) objective:
$$
\tilde{\mathcal{J}}_{\text{DSM}}(\theta)=\mathbb{E}_{t\sim\mathcal{U}(0, T)}\left\{\lambda(t) \mathbb{E}_{X_0\sim p_{data}(\cdot)} \mathbb{E}_{X_t \mid X_0}\left[\left|s_{\theta}(t,X(t))-\nabla_{X_t} \log p(t,X_t \mid X_0)\right|^2\right]\right\}
$$

\begin{theorem}
\label{thm:DSM}
Let 
$\mathcal{J}_{\text{DSM}}(\theta):= \mathbb{E}_{X_0\sim p_{data}(\cdot)} \mathbb{E}_{X_t \mid X_0}\left[\left|s_{\theta}(t,X(t))-\nabla_{X_t} \log p(t,X_t \mid X_0)\right|^2\right]$.
Under suitable conditions on $s_\theta$, we have 
$\mathcal{J}_{\text{DSM}}(\theta) = \mathcal{J}_{\text{ESM}}(\theta) + C$ for some $C$ independent of $\theta$. 
Consequently, the minimum point of $\mathcal{J}_{\text{DSM}}$ and that of $\mathcal{J}_{\text{ESM}}$ coincide. 
\end{theorem}
\begin{proof}
We have 
\begin{equation*}
\begin{aligned}
\mathcal{J}_{\text{ESM}}(\theta) 
& = \mathbb{E}_{p(t, \cdot)} |s_\theta(t,X) - \nabla \log p(t,X)|^2\\
& = \mathbb{E}_{p(t, \cdot)} \left[|s_\theta(t,X)|^2 - 2 s_\theta(t,X)^{\top}\nabla \log p(t,X)+|\nabla \log p(t,X)|^2\right].
\end{aligned}
\end{equation*}
Consider the inner product term, rewriting it as:
\begin{equation*}
\begin{aligned}
\mathbb{E}_{p(t, \cdot)} \left[s_\theta(t,X)^{\top}\nabla \log p(t,X)\right]& = \int_{x}p(t,x)s_\theta(t,x)^{\top}\nabla \log p(t,x) \mathrm{d}x\\
& = \int_{x}s_\theta(t,x)^{\top}\nabla p(t,x) \mathrm{d}x\\
& = \int_{x}
s_\theta(t,x)^{\top}\nabla 
\int_{x_0}p(0,x_0)p(t,x|x_0)\mathrm{d}x_0 \mathrm{d}x\\
& = \int_{x_0}\int_{x}
s_\theta(t,x)^{\top} 
p(0,x_0)\nabla p(t,x|x_0)\mathrm{d}x\mathrm{d}x_0 \\
& = \int_{x_0}p(0,x_0)\int_{x}
s_\theta(t,x)^{\top} 
p(t,x|x_0) \nabla \log p(t,x|x_0)\mathrm{d}x\mathrm{d}x_0\\
& = \mathbb{E}_{X_0\sim p(0,\cdot)} \mathbb{E}_{X_t \mid X_0} \left[s_{\theta}(t,X(t))^{\top}\nabla_{X_t} \log p(t,X_t \mid X_0)\right],
\end{aligned}
\end{equation*}
combining $\mathbb{E}_{X} |s_\theta(t,X)|^2=\mathbb{E}_{X_0} \mathbb{E}_{X \mid X_0}|s_\theta(t,X)|^2$ concludes our proof.
\end{proof}

\quad The intuition of DSM is that following the gradient $s_{\theta}$ of the log density at some corrupted point $\tilde{x}$ should ideally move us towards the clean sample $x$. The reason that the objective $\mathcal{J}_{DSM}$ is comparatively easy to solve is that conditional distribution usually satisfies a good distribution, like Gaussian kernel, i.e. $p(X_t \mid X_0) \sim N\left(X_t ; \mu_{t}(X_0), \sigma_t^2 I\right)$, which is satisfied in many cases of DPM, we can explicitly compute that:
$$
\nabla_{X_t} \log p(t,X_t \mid X_0)=\frac{1}{\sigma_t^2}(\mu_{t}(X_0)-X_t) .
$$
Direction $\frac{1}{\sigma_t^2}(X_0-\mu_{t}(X_0))$ clearly corresponds to moving from $\tilde{x}$ back towards clean sample $x$, and we want $s_{\theta}$ to match that as best it can. Moreover,  empirically validated by e.g. \cite{Song20}, a good candidate of $\lambda(t)$ is chosen as:
$$
\lambda(t) \propto 1 / \mathbb{E}\left[\left|\nabla_{X_t} \log p(t,X_t \mid X_0)\right|^2\right]=\sigma_t^2
$$
Thus, our final optimization objective is:
$$
\begin{aligned}
\tilde{\mathcal{J}}_{\text{DSM}}(\theta)&=\mathbb{E}_{t\sim\mathcal{U}(0, T)}\left\{\sigma_t^2 \mathbb{E}_{X_0\sim p_{data}(\cdot)} \mathbb{E}_{X_t \mid X_0}\left[\left|s_{\theta}(t,X(t))-\nabla_{X_t} \log p(t,X_t \mid X_0)\right|^2\right]\right\}\\
& = \mathbb{E}_{t\sim\mathcal{U}(0, T)}\left\{ \mathbb{E}_{X_0\sim p_{data}(\cdot)} \mathbb{E}_{X_t \mid X_0}\left[\left|\sigma_t s_{\theta}(t,X(t))+\frac{X_t-\mu_{t}(X_0)}{\sigma_t}\right|^2\right]\right\}\\
& = \mathbb{E}_{t\sim\mathcal{U}(0, T)}\left\{ \mathbb{E}_{X_0\sim p_{data}(\cdot)} \mathbb{E}_{\epsilon_t\sim\mathcal{N}(0,I)}\left[\left|\sigma_t s_{\theta}(t,\mu_{t}(X_0)+\sigma_t\epsilon_t)+\epsilon_t\right|^2\right]\right\}
\end{aligned}
$$
where the second equality holds when $X_t\mid X_0$ follows a conditionally normal and the third equality follows from a reparameterization/change of variables.

\subsection*{D. Proof of Theorem \ref{thm:sensitivity}}

The idea relies on coupling, which is similar to \cite[Lemma 4]{ZTY23}.
Consider the coupled SDEs:
\begin{equation*}
\left\{ \begin{array}{lcl}
d Y_t = \left(-b(T-t, Y_t) + \sigma^2(T-t) \nabla \log p(T-t, Y_t) \right) dt+ \sigma(T-t) dB_t, \\
d Z_t =  \left(-b(T-t, Z_t) + \sigma^2(T-t)s_\theta(T-t, Z_t) \right) dt+ \sigma(T-t) dB_t,
\end{array}\right.
\end{equation*}
where $(Y_0, Z_0)$ are coupled to achieve $W_2(p(T,\cdot), p_{\scalebox{0.7}{noise}}(\cdot))$,
i.e. $\mathbb{E}|Y_0 - Z_0|^2=W_2(p(T,\cdot), p_{\scalebox{0.7}{noise}}(\cdot))$. 
It is easy to see that
\begin{equation}
\label{eq:W2}
W_2^2(p_{\scalebox{0.7}{data}}(\cdot), \overline{X}_T) \le \mathbb{E}|Y_T - Z_T|^2.
\end{equation}
So the goal is to bound $\mathbb{E}|Y_T - Z_T|^2$. 
By It\^o's formula, we get
\begin{multline*}
d|Y_t - Z_t|^2 = 2(Y_t - Z_t) \cdot (-b(T-t, Y_t) + \sigma^2(T-t) \nabla \log p(T-t, Y_t) \\ + b(T-t, Z_t) - \sigma^2(T-t)s_\theta(T-t, Z_t)) dt
\end{multline*}
which implies that 
\begin{multline}
\label{eq:diff}
\frac{d \, \mathbb{E}|Y_t - Z_t|^2}{dt} 
= -2 \underbrace{\mathbb{E}((Y_t - Z_t) \cdot (b(T-t, Y_t) - b(T-t, Z_t))}_{(a)}) \\
+ 2 \underbrace{\mathbb{E}((Y_t - Z_t) \cdot \sigma^2(T-t) (\nabla \log p(T-t, Y_t) - s_\theta(T-t, Z_t)))}_{(b)}.
\end{multline}
By Assumption \ref{assump:sensitivity} (1), we get
\begin{equation}
\label{eq:terma}
(a) \ge r_b(T-t) \, \mathbb{E}|Y_t - Z_t|^2.
\end{equation}
Moreover,
\begin{equation}
\label{eq:termb}
\begin{aligned}
(b) & = \sigma^2(T-t) \bigg( \mathbb{E}((Y_t - Z_t) \cdot (\nabla \log p(T-t, Y_t) - \nabla \log p (T-t, Z_t))) \\
& \qquad \qquad \qquad \qquad \qquad \qquad \qquad \qquad + \mathbb{E}((Y_t - Z_t) \cdot (\nabla \log p (T-t, Z_t) - s_{\theta}(T-t, Z_t))) \bigg)\\
&  \le  \sigma^2(T-t) \left( L\, \mathbb{E}|Y_t - Z_t|^2 + h \mathbb{E}|Y_t - Z_t|^2 + \frac{1}{4h} \varepsilon^2\right), \\
\end{aligned}
\end{equation}
where we use Assumption \ref{assump:sensitivity} (2)(3) in the last inequality.
Combining \eqref{eq:diff}, \eqref{eq:terma} and \eqref{eq:termb}, we have
\begin{equation}
\label{eq:Gronwall}
\frac{d \, \mathbb{E}|Y_t - Z_t|^2}{dt} \le \left(-2 r_b(T-t) + (2h + 2L) \sigma^2(T-t) \right) \mathbb{E}|Y_t -Z_t|^2
+ \frac{\varepsilon^2}{2h} \sigma^2(T-t).
\end{equation}
Applying Gr\"onwall's inequality, we have:
\begin{equation*}
\mathbb{E}|Y_T - Z_T|^2\leq e^{u(T)}\mathbb{E}|Y_0 - Z_0|^2+ \frac{\varepsilon^2}{2h} \int_0^T \sigma^2(T-t) e^{u(T)-u(t)} dt,
\end{equation*}
which combined with \eqref{eq:W2} yields \eqref{eq:sensitivity}. 

\subsection*{E. Proof of Theorem \ref{thm:CVPW2}}

Recall that $r_b(t) = \frac{1}{2} \beta(t)$, $\sigma(t) = \sqrt{\beta(t)}$ and 
$p_{\tiny \mbox{noise}}(\cdot) \sim \mathcal{N}(0, (e^{\frac{T}{2} (\beta_{\max} + \beta_{\min})} - 1) I)$.
By \cite[Proposition 10]{GNZ23},
if $\log p_{\scalebox{0.7}{data}}(\cdot)$ is $\kappa$-strongly log-concave,
then $\nabla \log p(T-t, \cdot)$ is 
$\kappa \left(e^{\int_0^{T-t} \beta(s) ds}+ \kappa \int_0^{T-t}e^{\int_s^{T-t} \beta(v) dv} \beta(s) ds\right)^{-1}$-strongly concave.
Thus, the term $\mathbb{E}((Y_t - Z_t) \cdot (\nabla \log p(T-t, Y_t) - \nabla \log p(T-t, Z_t)))$ in \eqref{eq:termb} is bounded from above by
\begin{equation*}
    -\frac{\kappa}{e^{\int_0^{T-t} \beta(s) ds}+ \kappa \int_0^{T-t}e^{\int_s^{T-t} \beta(v) dv} \beta(s) ds} \mathbb{E}|Y_t - Z_t|^2,
\end{equation*}
instead of $L \mathbb{E}|Y_t - Z_t|^2$.
Consequently, 
the bound \eqref{eq:sensitivity} holds by replacing $u(t)$ with 
\begin{equation}
u_{\tiny \mbox{CVP}}(t): = \int_{T-t}^T \beta(s) \left(-1 + 2h - \frac{2 \kappa}{e^{\int_0^{s} \beta(v) dv}+ \kappa \int_0^{s}e^{ \int_v^{s} \beta(u) du} \beta(v) dv}\right) ds.
\end{equation}
Note that 
$u_{\tiny \mbox{CVP}}(T) \le - \int_{0}^T \beta(s) ds + 2  \beta_{\max} h T
- \frac{2 \kappa}{\kappa+1} \left(1 - e^{-\beta_{\min} T} \right)$
and 
$u_{\tiny \mbox{CVP}}(T) - u_{\tiny \mbox{CVP}}(T-t) \le \beta_{\max}(2h-1) t$.
Moreover, 
$W_2^2(p(T, \cdot), p_{\tiny \mbox{noise}}(\cdot)) 
\le e^{\int_0^T \beta(s)ds} \mathbb{E}_{p_{\tiny \mbox{data}}(\cdot)}|x|^2$.
Combining \eqref{eq:sensitivity} with the above estimates
yields \eqref{eq:sensitivity3}.

\subsection*{F. Proof of Theorem \ref{thm:globalerr}}

The analysis of the error $\left(\mathbb{E}|\overline{X}_T - \widehat{X}_N|^2\right)^{\frac{1}{2}}$
relies on the following lemmas. 
The lemma below proves the contraction of the backward SDE $\overline{X}$.
\begin{lemma}
\label{lem:contraction}
Let $(\overline{X}^x_t, \, 0 \le t \le T)$
be defined by \eqref{eq:SDErevc} with $\overline{X}^x_0 = x$.
Let Assumptions \ref{assump:sensitivity}, \ref{assump:2dis} and \ref{assump:contraction} hold.
Then
\begin{equation}
\label{eq:contraction}
\left(\mathbb{E}|\overline{X}^x_t - \overline{X}^y_t|^2\right)^{\frac{1}{2}}
\le \left(\mathbb{E}|x - y|^2\right)^{\frac{1}{2}} \exp(- 2 \beta t), \quad \mbox{for all } t,
\end{equation}
where $\overline{X}^x$ and $\overline{X}^y$ be coupled, i.e.
they are driven by the same Brownian motion with (different) initial values $x$ and $y$ respectively ($x$ and $y$ represent two random variables).
\end{lemma}
\begin{proof}
Note that
\begin{multline*}
d|\overline{X}^x_s - \overline{X}^y_s|^2
= 2 \left(\overline{X}^x_s - \overline{X}^y_s \right)\cdot
\bigg(-b(T-s, \overline{X}^x_s) + \sigma^2(T-s) s_\theta(T-s, \overline{X}^x_s)\\
+ b(T-s, \overline{X}^y_s) - \sigma^2(T-s) s_\theta(\overline{X}^y_s)\bigg) ds.
\end{multline*}
Thus, 
\begin{equation}
\label{eq:diffiden}
\begin{aligned}
\frac{d}{ds} \mathbb{E}|\overline{X}^x_s - \overline{X}^y_s|^2
= -2 & \underbrace{\mathbb{E}\left[(\overline{X}^x_s - \overline{X}^y_s) \cdot (b(T-s, \overline{X}^x_s) - b(T-s, \overline{X}^y_s))\right]}_{(a)} \\
& + 2 \underbrace{\mathbb{E}\left[(\overline{X}^x_s - \overline{X}^y_s) \sigma^2(T-s) 
(s_\theta(T-s, \overline{X}^x_s) - s_\theta(T-s,\overline{X}^y_s))\right]}_{(b)}.
\end{aligned}
\end{equation}
By Assumption \ref{assump:sensitivity} (1),
we get
\begin{equation}
\label{eq:terma_diff}
(a) \ge  r_b(T-s) \mathbb{E}|\overline{X}^x_s - \overline{X}^y_s|^2.
\end{equation}
By Assumption \ref{assump:2dis} (4), we obtain
\begin{equation}
\label{eq:termb_diff}
(b) \le L_s \sigma^2(T-s) \mathbb{E}|\overline{X}^x_s - \overline{X}^y_s|^2.
\end{equation}
Combining \eqref{eq:diffiden}, \eqref{eq:terma_diff} 
and \eqref{eq:termb_diff} yields
\begin{equation*}
\frac{d}{ds} \mathbb{E}|\overline{X}^x_s - \overline{X}^y_s|^2 
\le -2 (r_b(T-s) - L_s \sigma^2(T-s)) \mathbb{E}|\overline{X}^x_s - \overline{X}^y_s|^2.
\end{equation*}
By Gr\"onwall's inequality, we have
\begin{equation*}
\mathbb{E}|\overline{X}^x_s - \overline{X}^y_s|^2 
\le \mathbb{E}|x - y|^2 \exp \left(-2 \int_0^t (r_b(T-s) - L_s \sigma^2(T-s)) ds  \right),
\end{equation*}
which, by the condition \eqref{eq:condcontraction}, yields \eqref{eq:contraction}
\end{proof}

Next we deal with the {\em local (one-step) discretization error} of the process $\overline{X}$.
Fixing $t_{\star} \le T-\delta$, 
the (one-step) discretization of $\overline{X}$ starting at $\overline{X}_{t_\star} = x$
is:
\begin{equation}
\label{eq:localEM}
\begin{aligned}
\widehat{X}^{t_\star, x}_1 = x + (-b(T-t_\star, x) + 
\sigma^2(T-t_\star) s_\theta(T-t_\star, x)) \delta + \sigma(T - t_\star) (B_{t_\star + \delta} - B_{t_\star}). 
\end{aligned}
\end{equation}
The following lemma provides an estimate of the local discretization error.
\begin{lemma}
\label{lem:localest}
Let $(\overline{X}_t^{t_\star, x}, \, t_\star \le t \le T)$ be defined by \eqref{eq:SDErevc} with $\overline{X}^{t_\star, x}_{t_\star} = x$,
and $\widehat{X}^{t_\star, x}_1$ be given by \eqref{eq:localEM}.
Let Assumption \ref{assump:2dis} hold. 
Then for $\delta$ sufficiently small (i.e. $\delta \le \overline{\delta}$ for some $\overline{\delta} < 1$),
there exists $C_1, C_2 > 0$ independent of $\delta$ and $x$ such that
\begin{equation}
\label{eq:localest1}
\left(\mathbb{E}|\overline{X}^{t_\star, x}_{t_{\star} + \delta} -\widehat{X}^{t_\star, x}_1 |^2 \right)^{\frac{1}{2}}
\le (C_1 + C_2 \sqrt{\mathbb{E}|x|^2} )^{\frac{1}{2}} \delta^{\frac{3}{2}},
\end{equation}
\begin{equation}
\label{eq:localest2}
|\mathbb{E}(\overline{X}^{t_\star, x}_{t_{\star} + \delta} -\widehat{X}^{t_\star, x}_1)|
\le (C_1 + C_2 \sqrt{\mathbb{E}|x|^2} )^{\frac{1}{2}} \delta^{\frac{3}{2}}.
\end{equation}
\end{lemma}
\begin{proof}
For ease of presentation, 
we write $\overline{X}_{t}$ (resp. $\widehat{X}_1$) 
for $\overline{X}^{t_\star, x}_t$ (resp. $\widehat{X}^{t_\star, x}_1$).
Without loss of generality, set $t_\star = 0$.
We have
\begin{equation*}
\begin{aligned}
& \overline{X}_\delta = x + \int_0^\delta -b(T-t, \overline{X}_t) + \sigma^2(T-t) s_\theta(T-t, \overline{X}_t) dt + \int_0^\delta \sigma(T-t) dB_t, \\
& \widehat{X}_1 = x + \int_0^\delta -b(T,x) + \sigma^2(T) s_\theta(T,x) dt
+ \int_0^\delta \sigma(T) dB_t.
\end{aligned}
\end{equation*}
So
\begin{equation}
\label{eq:diffthree}
\begin{aligned}
& \mathbb{E}|\overline{X}_\delta -\widehat{X}_1|^2  \\
& = \mathbb{E} \bigg|\int_0^\delta b(T,x) - b(T-t, \overline{X}_t) dt + 
\int_0^\delta \sigma^2(T-t) s_\theta(T-t, \overline{X}_t) -\sigma^2(T) s_\theta(T,x) dt  \\
& \qquad \qquad \qquad \qquad \qquad \qquad \qquad \qquad \qquad + \int_0^\delta \sigma (T-t) - \sigma(T) dB_t \bigg|^2 \\
&\le 3 \mathbb{E} \bigg( \left|\int_0^\delta b(T,x) - b(T-t, \overline{X}_t) dt\right|^2
+ \left| \int_0^\delta \sigma^2(T-t) s_\theta(T-t, \overline{X}_t) -\sigma^2(T) s_\theta(T,x) dt\right|^2 \\
&  \qquad \qquad \qquad \qquad \qquad \qquad \qquad \qquad \qquad +
\left| \int_0^\delta \sigma (T-t) - \sigma(T) dB_t \right|^2 \bigg) \\
& \le 3 \bigg(\delta  \underbrace{\int_0^\delta \mathbb{E} |b(T,x) - b(T-t, \overline{X}_t)|^2 dt}_{(a)} + \delta \underbrace{\int_0^\delta \mathbb{E} |\sigma^2(T-t) s_\theta(T-t, \overline{X}_t) -\sigma^2(T) s_\theta(T,x)|^2 dt}_{(b)} \\
& \qquad \qquad \qquad \qquad \qquad \qquad \qquad \qquad \qquad +
\underbrace{\int_0^\delta |\sigma(T-t) - \sigma(T)|^2 dt}_{(c)} \bigg),
\end{aligned}
\end{equation}
where we use the Cauchy–Schwarz inequality and It\^o's isometry in the last inequality. 
By Assumption \ref{assump:2dis} (1), we get
\begin{equation}
\label{eq:diffthreec}
(c) \le \int_0^\delta L_\sigma^2 t^2 dt = \frac{L_\sigma^2}{3} \delta^3.
\end{equation}
By Assumption \ref{assump:2dis} (3), we have
\begin{equation*}
(a) \le \int_0^\delta 2 L_b^2 (t^2 + \mathbb{E}|\overline{X}_t - x|^2) dt
= 2L_b^2 \left(\frac{\delta^3}{3} + \int_0^\delta \mathbb{E}|\overline{X}_t - x|^2 dt \right).
\end{equation*}
According to \cite[Theorem 4.5.4]{KP92},
we have $\mathbb{E}|\overline{X}_t - x|^2 \le C(1 + \mathbb{E}|x|^2)t e^{Ct}$ 
for some $C>0$ (independent of $x$).
Consequently, for $t \le \delta$ sufficiently small (bounded by $\overline{\delta} < 1$),
\begin{equation*}
\mathbb{E}|\overline{X}_t - x|^2 \le C'(1 + \mathbb{E}|x|^2) t, \quad \mbox{for some } C' > 0 \mbox{ (independent of } \delta, x).
\end{equation*}
We then get
\begin{equation}
\label{eq:diffthreea}
(a) \le 2 L_b^2 \left(\frac{\delta^3}{3} + \frac{C'(1 + \mathbb{E}|x|^2)}{2} \delta^2  \right) \le 2 L_b^2 \left(\frac{1}{3} + \frac{C'}{2} + \frac{C'}{2} \mathbb{E}|x|^2 \right) \delta^2.
\end{equation}
Similarly, we obtain by Assumption \ref{assump:2dis} (1)(2)(4)(5):
\begin{equation}
\label{eq:diffthreeb}
(b) \le C'' (1 + \mathbb{E}|x|^2) \delta^2, \quad \mbox{for some } C''>0 \mbox{ (independent of } \delta, x).
\end{equation}
Combining \eqref{eq:diffthree}, \eqref{eq:diffthreec},
\eqref{eq:diffthreea} and \eqref{eq:diffthreeb} yields
the estimate \eqref{eq:localest1}.

Next we have
\begin{align*}
& |\mathbb{E}(\overline{X}_\delta - \widehat{X}_1)|  \\
& = \left|\mathbb{E}\int_0^\delta b(T,x) - b(T-t,\overline{X}_t) dt
+ \mathbb{E}\int_0^\delta \sigma^2(T-t) s_\theta(T-t, \overline{X}_t) - \sigma^2(T)s_\theta(T,x) dt \right| \\
& \le \int_0^\delta \mathbb{E}|b(T,x) - b(T-t,\overline{X}_t)| dt
+ \int_0^\delta \mathbb{E}|\sigma^2(T-t) s_\theta(T-t, \overline{X}_t) - \sigma^2(T)s_\theta(T,x)| dt \\
& \le C''' \int_0^\delta  \left(t (1 + \mathbb{E}|x|) + \mathbb{E}|\overline{X}_t - x| \right) dt \\
& \le C''''(1 + \sqrt{\mathbb{E}|x|^2}) \delta^{\frac{3}{2}},
\quad \mbox{for some } C''''>0 \mbox{ (independent of } \delta, x).
\end{align*}
where the third inequality follows from Assumption \ref{assump:2dis}, and the last inequality is due to the fact that
$\mathbb{E}|\overline{X}_t - x| \le \left(\mathbb{E}|\overline{X}_t - x|^2\right)^{\frac{1}{2}} \le \sqrt{C'(1 + \mathbb{E}|x|^2)t}$.
This yields the estimate \eqref{eq:localest2}.
\end{proof}

\begin{proof}[Proof of Theorem \ref{thm:globalerr}]
The proof is split into four steps.

\smallskip
{\bf Step 1}. Recall that $t_k = k \delta$ for $k = 0, \ldots, N$.
Denote $\overline{X}_k:=\overline{X}_{t_k}$,
and let 
\begin{equation*}
e_k: = \left(\mathbb{E}|\overline{X}_k - \widehat{X}_k|^2 \right)^{\frac{1}{2}}.
\end{equation*}
The idea is to build a recursion for the sequence $(e_k)_{k = 0, \ldots, N}$.
Also write $(\overline{X}^{t_\star, x}_t, \, t_\star \le t \le T)$ to emphasize that
the reversed SDE \eqref{eq:SDErevc} starts at $\overline{X}^{t_\star, x}_{t_\star} = x$,
so $\overline{X}_{k+1} = \overline{X}^{t_k, \overline{X}_k}_{t_{k+1}}$.
We have
\begin{equation}
\label{eq:decompabc}
\begin{aligned}
e_{k+1}^2 & = \mathbb{E}\left|\overline{X}_{k+1} - \overline{X}_{t_{k+1}}^{t_k, \widehat{X}_k} +  \overline{X}_{t_{k+1}}^{t_k, \widehat{X}_k}- \widehat{X}_{k+1} \right|^2 \\
& = \underbrace{\mathbb{E}|\overline{X}_{k+1} - \overline{X}_{t_{k+1}}^{t_k, \widehat{X}_k}|^2}_{(a)} + \underbrace{\mathbb{E}|\overline{X}_{t_{k+1}}^{t_k, \widehat{X}_k}- \widehat{X}_{k+1}|^2}_{(b)} + 2 \underbrace{\mathbb{E}\left[ (\overline{X}_{k+1} - \overline{X}_{t_{k+1}}^{t_k, \widehat{X}_k}) (\overline{X}_{t_{k+1}}^{t_k, \widehat{X}_{k}}- \widehat{X}_{k+1})\right]}_{(c)}.
\end{aligned}
\end{equation}

{\bf Step 2}. We analyze the term (a) and (b). 
By Lemma \ref{lem:contraction} (the contraction property), 
we get
\begin{equation}
\label{eq:terma1}
(a) = \mathbb{E}|\overline{X}_{t_{k+1}}^{t_k, \overline{X}_k} - \overline{X}_{t_{k+1}}^{t_k, \widehat{X}_k}|^2 \le e_k^2 \exp(- 2 \beta \delta).
\end{equation}
By \eqref{eq:localest1} (in Lemma \ref{lem:localest}), we have
\begin{equation}
\label{eq:termb1}
(b) \le \left(C_1 + C_2 \mathbb{E}|\widehat{X}_k|^2 \right) \delta^3.
\end{equation}

{\bf Step 3}. We analyze the cross-product (c). 
By splitting
\begin{equation*}
\overline{X}_{k+1} - \overline{X}_{t_{k+1}}^{t_k, \widehat{X}_k}
= (\overline{X}_k - \widehat{X}_k) + \underbrace{\left[ (\overline{X}_{k+1} - \overline{X}_k) - (\overline{X}_{t_{k+1}}^{t_k, \widehat{X}_k} - \widehat{X}_k)\right]}_{:= d_\delta(\overline{X}_k, \widehat{X}_k)},
\end{equation*}
we obtain
\begin{equation}
\label{eq:termc2}
(c) = \underbrace{\mathbb{E}\left[(\overline{X}_k - \widehat{X}_k) (\overline{X}_{t_{k+1}}^{t_k, \widehat{X}_{k}}- \widehat{X}_{k+1})\right]}_{(d)} + \underbrace{\mathbb{E}\left[d_\delta(\overline{X}_k, \widehat{X}_k) (\overline{X}_{t_{k+1}}^{t_k, \widehat{X}_{k}}- \widehat{X}_{k+1})\right]}_{(e)}.
\end{equation}
For the term (d), we have
\begin{equation}
\label{eq:termd}
\begin{aligned}
(d)& = \mathbb{E}\left[(\overline{X}_k - \widehat{X}_k) \, \mathbb{E}(\overline{X}_{t_{k+1}}^{t_k, \widehat{X}_{k}}- \widehat{X}_{k+1}| \mathcal{F}_k)\right] \\
& \le e_k \left(\mathbb{E}|\mathbb{E}(\overline{X}_{t_{k+1}}^{t_k, \widehat{X}_{k}}- \widehat{X}_{k+1}| \mathcal{F}_k)|^2 \right)^{\frac{1}{2}} \\
& \le e_k \left(C_1 + C_2 \sqrt{\mathbb{E}|\widehat{X}_k|^2} \right) \delta^{\frac{3}{2}},
\end{aligned}
\end{equation}
where we use the tower property (of the conditional expectation) in the first equation, the Cauchy-Schwarz inequality in the second inequality, and \eqref{eq:localest2} in the final inequality. 
According to \cite[Lemma 1.3]{MT04}, there exists $C_0 > 0$ (independent of $\delta, \widehat{X}_k$) such that
\begin{equation}
\label{eq:dest}
\left(\mathbb{E}d^2_\delta(\overline{X}_k, \widehat{X}_k)\right)^{\frac{1}{2}} \le C_0 e_k \sqrt{\delta}.
\end{equation}
Thus,
\begin{equation}
\label{eq:terme}
\begin{aligned}
(e) & \le \left( \mathbb{E} d^2_\delta(\overline{X}_k, \widehat{X}_k \right)^{\frac{1}{2}} \left( \mathbb{E}|\overline{X}_{t_{k+1}}^{t_k, \widehat{X}_{k}}- \widehat{X}_{k+1}|^2\right)^{\frac{1}{2}} \\
& \le C_0 e_k \left(C_1 + C_2 \sqrt{\mathbb{E}|\widehat{X}_k|^2} \right)\delta^2.
\end{aligned}
\end{equation}
where we use \eqref{eq:localest1} and \eqref{eq:dest} in the last inequality.
Combining \eqref{eq:termc2}, \eqref{eq:termd} and \eqref{eq:terme} yields
for $\delta$ sufficiently small,
\begin{equation}
\label{eq:termc}
(c) \le e_k \left(C'_1 + C'_2 \sqrt{\mathbb{E}|\widehat{X}_k|^2} \right) \delta^{\frac{3}{2}}, \quad \mbox{for some } C'_1, C'_2 > 0 \mbox{ (independent of } \delta, \widehat{X}_k).
\end{equation}

{\bf Step 4}. Combining \eqref{eq:decompabc} with \eqref{eq:terma1}, \eqref{eq:termb1} and \eqref{eq:termc} yields
\begin{equation*}
e_{k+1}^2 \le e_k^2 \exp(- 2 \beta \delta) + \left( C_1 + C_2 \mathbb{E}|\widehat{X}_k|^2 \right) \delta^3 + e_k \left(C'_1 + C'_2 \sqrt{\mathbb{E}|\widehat{X}_k|^2} \right) \delta^{\frac{3}{2}}.
\end{equation*}
A standard argument shows that Lemma \ref{lem:contraction} (the contraction property) implies
$\mathbb{E}|\overline{X}_t|^2 \le C$ for some $C > 0$.
Thus, $\mathbb{E}|\widehat{X}_k|^2 \le 2(C + e_k^2)$.
As a result, for $\delta$ sufficiently small,
\begin{equation}
\label{eq:rec1}
e_{k+1}^2 \le e_k^2\left(1 - \frac{3}{4} \beta \delta \right)
+ D_1 \delta^3 + D_2 e_k^2 \left(\delta^3 + \delta^{\frac{3}{2}} \right)  + D_3 e_k \delta^{\frac{3}{2}},
\end{equation}
for some $D_1, D_2, D_3> 0$ (independent of $\delta$).
Note that 
\begin{equation*}
D_2 e_k^2 \left(\delta^3 + \delta^{\frac{3}{2}} \right) \le \frac{1}{4} e_k^2 \beta \delta, \quad \mbox{ for } \delta \mbox{ sufficiently small},
\end{equation*}
and 
\begin{equation*}
D_3 e_k \delta^{\frac{3}{2}} \le \frac{1}{4}e_k^2 \beta \delta + \frac{2 D^{2}_3}{\beta} \delta^2.
\end{equation*}
Thus, the estimate \eqref{eq:rec1} leads to
\begin{equation}
\label{eq:rec2}
e_{k+1}^2 \le e_k^2 \left(1 - \frac{1}{4} \beta \delta \right) + D \delta^2, \quad 
\mbox{for some } D > 0 \mbox{ (independent of } \delta).
\end{equation}
Unfolding the inequality \eqref{eq:rec2} yields
the estimate \eqref{eq:globalerr}.
\end{proof}

\quad As a remark, if we can improve the estimate in \eqref{eq:localest2} to
\begin{equation}
|\mathbb{E}(\overline{X}^{t_\star, x}_{t_{\star} + \delta} -\widehat{X}^{t_\star, x}_1)|
\le (C_1 + C_2 \sqrt{\mathbb{E}|x|^2} )^{\frac{1}{2}} \delta^{2},
\end{equation} 
(i.e. $\delta^2$ local error instead of $\delta^{\frac{3}{2}}$),
then the discretization error is $C \delta$.

\subsection*{G. Experimental Details}
We provide implementation details of the experiments in Section \ref{sc4}. 
All codes will soon be available at github.

\subsection*{G.1. 1-dimensional data experiments}
For this 1-dimensional experiments, we start from a single point mass, 
and the parameters of the DPMs (OU and COU) are chosen as $\theta\equiv 0.2$, $\sigma(t)\equiv 0.5$. 

\subsection*{G.2. Swiss Roll and MNIST experiments}
In both experiments, 
we use contractive subVP with predictor-corrector sampler. 
We recommend to use $\beta_{min}=0.01$ and $\beta_{max}=8$ for a good result. The signal-noise-ratio is set to be $0.2$ for Swiss Roll,
and $0.1$ for MNIST. 
The Jupyter Notebooks can be found in the zip file of supplementary materials.

\subsection*{G.3. CIFAR10 experiments}
We use contractive subVP with predictor-corrector sampler. with all the other settings the same as VE SDE in \cite{Song20}. We recommend to use $\beta_{min}=0.01$ and $\beta_{max}=8$ for a good result.
The signal-noise-ratio is set to be 0.11 for the best result. 
The experiments based on NCSN++ \cite{Song20} are as below in Table 5.

\quad For EDM \cite{karras2022elucidating} with contraction, we adopt $f(0)=0.98$ to yield the best result by adding contraction, i.e. $\epsilon=0.02$. We find that $f(0)=0.97$ or $0.99$ also leads to better results than original EDM. Our denoising sampler is based on the deterministic sampler ($2^{\text{nd}}$ order scheme) of EDM paper and yields the same settings of sampling steps or other sampler constants. For computational resource, we used 4 L40S GPUs, and these results could readily be realized within 0.5 hours for each separate task.

\begin{table}[!htbp]
\centering
\begin{tabular}{l c c c c}
    \toprule
        Parameter & Value\\
         \midrule
        $\beta_{min} $ & $0.01$\\
        $\beta_{max} $ & $8$\\
        $snr$ & $0.11$\\
        $\operatorname{grad}_{clip} $ & $10$\\
        $\operatorname{\alpha_{lr}}$ & $5e^{-4}$\\
    \bottomrule
\end{tabular}
\vspace{5 pt}
\label{tab:CIFAR10 details}
\captionof{table}{Parameters of the CDPM for CIFAR-10.}
\end{table}

\subsection*{G.4 More examples of CIFAR10 synthesis by CsubVP}
Here we provide more examples of CIFAR10 32$\times$32 synthesis by CsubVP SDEs in Figure \ref{fig:More CIFAR10syn}.
\begin{figure}
\vspace{2pt}
    \centering
    \includegraphics[width=1\linewidth]{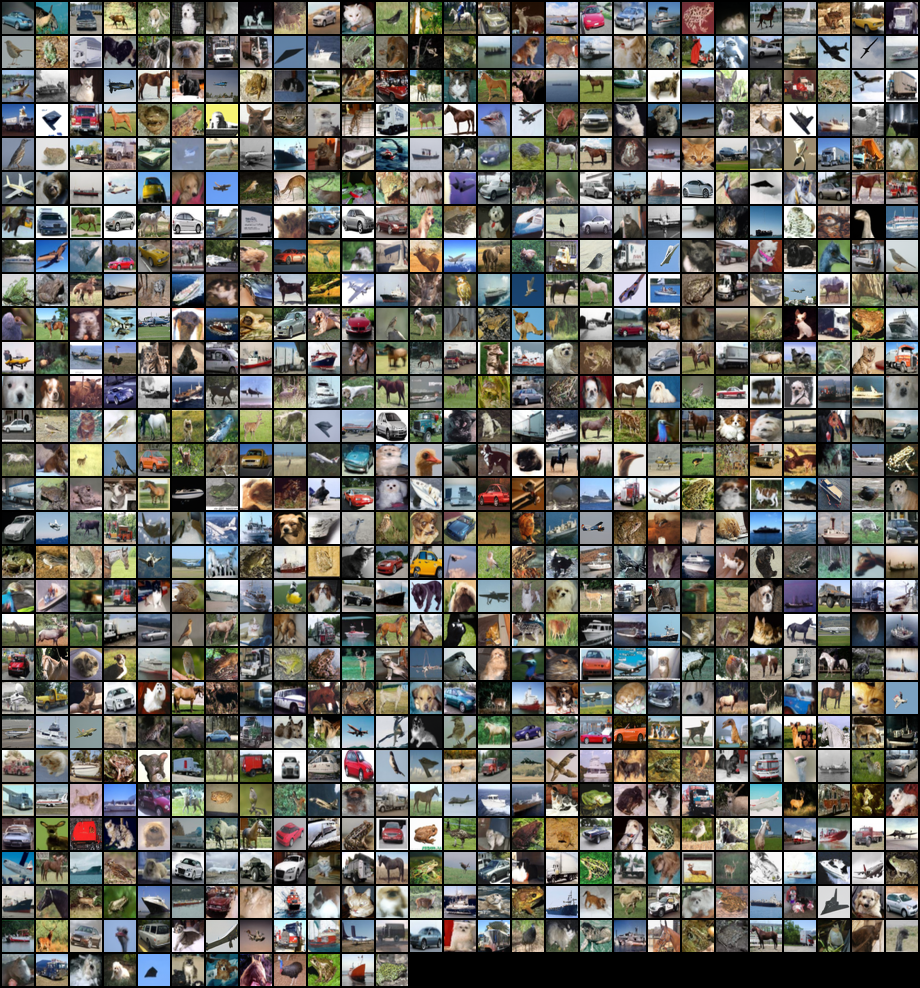}
    \captionof{figure}{CsubVP CIFAR10 samples}
    \label{fig:More CIFAR10syn}
\end{figure}

\end{document}